\documentclass{article}
\usepackage[accepted]{icml2025}

\usepackage{hyperref}

\usepackage{microtype}
\usepackage{graphicx}
\usepackage{subfigure}
\usepackage{booktabs} % for professional tables

\usepackage{amsmath}
\usepackage{amssymb}
\usepackage{mathtools}
\usepackage{amsthm}

\usepackage[capitalize,noabbrev]{cleveref}

\usepackage{bbm}
\usepackage{multirow}
\usepackage{paralist}
\usepackage{tcolorbox}
\usepackage{xargs}
\usepackage{dsfont}
\usepackage{subcaption} % graphicx,
\usepackage{enumitem}
\usepackage{aliascnt}
\usepackage{algorithm}
\usepackage{xspace}
\usepackage{csquotes}

\theoremstyle{plain}
\newtheorem{theorem}{Theorem}
\newtheorem{proposition}{Proposition}
\newtheorem{lemma}{Lemma}

\theoremstyle{definition}

\newtheorem{assumption}{\textbf{H}\hspace{-3pt}}
\crefformat{assumption}{{\textbf{H}}#2#1#3}
\theoremstyle{remark}

\theoremstyle{definition}
\newtheorem{example}{Example}

\definecolor{aurometalsaurus}{rgb}{0.43, 0.5, 0.5}
\definecolor{britishracinggreen}{rgb}{0.0, 0.26, 0.15}
\definecolor{burntumber}{rgb}{0.54, 0.2, 0.14}
\definecolor{cobalt}{rgb}{0.0, 0.28, 0.67}
\definecolor{bulgarianrose}{rgb}{0.28, 0.02, 0.03}
\definecolor{ceruleanblue}{rgb}{0.16, 0.32, 0.75}

\usepackage[color=blue!20]{todonotes}

\graphicspath{{fig/}}

\usepackage{xcolor} % define colors in text
\usepackage{braket} % For \Set

\newcommand{\group}[1]{\bgroup\small\sffamily\noindent\color{green}{#1}\egroup}  % pour commenter facilement des sections.
\newcommand{\new}[1]{\bgroup\small\sffamily\noindent\color{cobalt}{#1}\egroup}
\newcommand{\old}[1]{\bgroup\small\color{gray}{#1}\egroup}  % pour commenter facilement des sections.
\definecolor{darkgreen}{RGB}{0,128,0}

\newcommand{\R}{\mathbb{R}}

\newcommand{\E}{\mathbb{E}}

\newcommand{\var}{\operatorname{Var}}

\newcommand{\prob}{\mathbb{P}}

\newcommand{\rmd}{\mathrm{d}}
\newcommand{\1}{\mathds{1}}

\newcommand{\pr}[1]{\left({#1}\right)}

\newcommand{\ac}[1]{\left\{{#1}\right\}}

\newcommand{\abs}[1]{\left\lvert{#1}\right\rvert}

\newcommand{\gauss}{\mathcal{N}}

\newcommand{\q}[1]{Q_{#1}}

\renewcommand{\algorithmiccomment}[1]{\textbf{\textcolor{darkgreen}{// \texttt{#1}}}}

\newcommand{\tcount}{n}
\newcommand{\ccount}{m}
\newcommand{\XC}{\mathcal{X}}
\newcommand{\YC}{\mathcal{Y}}

\newcommand{\RCP}{\ensuremath{\texttt{RCP}}\xspace}

\newcommand{\adj}[1]{f_{#1}}
\newcommand{\adjinv}{\tilde{f}_{\varphi}}

\newcommand{\paragraphformat}[1]{\vspace{.2em}\noindent\textbf{#1}\hspace{0.2em}}

\let\mc\mathcal                                             % FANCY FONT (Big-Oh)
                                                 % ...
                                                % ...
\def\PE\mathbb{E}

\icmltitlerunning{Rectifying Conformity Scores for Better Conditional Coverage}

\begin{document}

% !TEX root = main.tex

\twocolumn[
\icmltitle{Rectifying Conformity Scores for Better Conditional Coverage}
% \icmltitle{Generalized Conformalized Quantile Regression: \\A New Approach for Better Conditional Coverage Refined Conformal Prediction for Better Conditional Coverage}

% It is OKAY to include author information, even for blind
% submissions: the style file will automatically remove it for you
% unless you've provided the [accepted] option to the icml2025
% package.

% List of affiliations: The first argument should be a (short)
% identifier you will use later to specify author affiliations
% Academic affiliations should list Department, University, City, Region, Country
% Industry affiliations should list Company, City, Region, Country

% You can specify symbols, otherwise they are numbered in order.
% Ideally, you should not use this facility. Affiliations will be numbered
% in order of appearance and this is the preferred way.
\icmlsetsymbol{equal}{*}

\begin{icmlauthorlist}
  \icmlauthor{Vincent Plassier}{equal,lagrange}
  \icmlauthor{Alexander Fishkov}{equal,mbzuai,skoltech}
  \icmlauthor{Victor Dheur}{equal,mons}
  \icmlauthor{Mohsen Guizani}{mbzuai}
  \icmlauthor{Souhaib Ben Taieb}{mbzuai,mons}
  \icmlauthor{Maxim Panov}{mbzuai}
  \icmlauthor{Eric Moulines}{mbzuai,ep}
\end{icmlauthorlist}

\icmlaffiliation{mbzuai}{Mohamed bin Zayed University of Artificial Intelligence}
\icmlaffiliation{skoltech}{Skolkovo Institute of Science and Technology}
\icmlaffiliation{mons}{University of Mons}
\icmlaffiliation{ep}{École Polytechnique}
\icmlaffiliation{lagrange}{Lagrange Mathematics and Computing Research Center}
\icmlcorrespondingauthor{Maxim Panov}{maxim.panov@mbzuai.ac.ae}

% You may provide any keywords that you
% find helpful for describing your paper; these are used to populate
% the "keywords" metadata in the PDF but will not be shown in the document
\icmlkeywords{Machine Learning, ICML}

\vskip 0.3in
]

% this must go after the closing bracket ] following \twocolumn[ ...

% This command actually creates the footnote in the first column
% listing the affiliations and the copyright notice.
% The command takes one argument, which is text to display at the start of the footnote.
% The \icmlEqualContribution command is standard text for equal contribution.
% Remove it (just {}) if you do not need this facility.

%\printAffiliationsAndNotice{}  % leave blank if no need to mention equal contribution
\printAffiliationsAndNotice{\icmlEqualContribution} % otherwise use the standard text.

\begin{abstract}
  We present a new method for generating confidence sets within the split conformal prediction framework. Our method performs a trainable transformation of a conformity score to improve conditional coverage with finite-sample marginal validity. The transformation is based on an estimate of the conditional quantile of conformity scores. The resulting method is particularly beneficial for constructing adaptive confidence sets in multi-output problems where standard conformal quantile regression approaches have limited applicability. We bound the conditional coverage error in terms of the accuracy of the quantile estimate and the size of the final calibration sample. We experimentally show that our method adapts to the local data structure and improves conditional coverage on the benchmarks studied.
\end{abstract}

\section{Introduction}
\label{sec:introduction}
  % !TEX root = ../main.tex

  The widespread deployment of  AI models emphasizes the need for reliable uncertainty quantification~\cite{gruber2023sourcesuncertaintymachinelearning}. Although highly flexible in capturing complex statistical dependencies, these models can produce unreliable or overly confident predictions~\citep{Nalisnick2018-ew}. Conformal prediction (CP; \citet{vovk2005algorithmic,shafer2008tutorial}) offers a robust, distribution-free framework for predictions with finite-sample validity guarantees~\citep{angelopoulos2023conformal,Angelopoulos2024-dr}. 
  %While most CP research targets univariate tasks~\citep{romano2019conformalized,sesia2021conformal,rossellini2024integrating}, multivariate scenarios remain less explored. Existing multivariate methods either aggregate univariate prediction regions~\citep{messoudi2021copula,Zhou2024-sq}, employ generative models~\citep{wang2023probabilistic,plassier2024conditionally}, or estimate conditional density functions~\citep{izbicki2022cd}.

  Classical CP approaches guarantee marginal validity but fail to ensure the more desirable property of conditional validity, which customizes prediction regions to specific covariates. Prior studies have shown constructing meaningful prediction regions with exact conditional validity is infeasible without additional distributional assumptions~\citep{vovk2012conditional,lei2014distribution,foygel2021limits}. Consequently, current research emphasizes developing conformal methods that maintain marginal validity and achieve \textit{approximate} conditional validity~\cite{colombo2024normalizing,gibbs2025conformal}.

  A typical relaxation of exact conditional coverage in earlier work involves group-conditional guarantees~\cite{jung2022batch,ding2024class}, which provide coverage guarantees for a predefined set of groups. Another branch of work partitions the covariate space $\XC$ into multiple regions and applies CP within each set in the partition~\cite{leroy2021md, alaa2023conformalized, kiyani2024conformal}. However, such partitioning based on the calibration set often leads to overly large prediction regions~\cite{bian2023training,plassier2024efficient}.

  An alternative approach weights the empirical cumulative distribution function with a ``localizer'' function that quantifies the similarity between calibration points and the test sample~\cite{guan2023localized}. Although this method improves the localization of predictions, it has significant limitations, especially in high-dimensional covariate spaces.

  Finally, several methods focus on the transformation of conformity scores~\cite{han2022split,dey2022conditionally,izbicki2022cd,deutschmann2023adaptive,dheur2024distribution,colombo2024normalizing}. These techniques adjust conformity scores to better approximate the conditional coverage. However, they usually require estimating the conditional distribution of conformity scores, which is both computationally intensive and difficult to perform accurately.

  In this paper, we propose a novel CP method, \textit{Rectified Conformal Prediction} (\RCP), extending normalized nonconformity scores; see, e.g.,~\cite{papadopoulos2008normalized,papadopoulos2011reliable}. \RCP aims to enhance conditional coverage while preserving exact marginal coverage guarantees. It uses an estimate of the conditional score quantile to make the rectified quantile approximately constant across covariates.

  A significant benefit of \RCP is its capacity to generate  prediction sets without fully modeling the conditional distribution of conformity scores. Instead, \RCP concentrates on quantile regression to ensure approximate conditional coverage. The main \textbf{contributions} of this work can be summarized as follows.
  \begin{itemize}
    \item We introduce Rectified Conformal Prediction (\RCP), a new conformal method designed to enhance conditional validity by refining conformity scores (see Sections~\ref{sec:rcp} and~\ref{sec:implementation_of_rcp}). The proposed method avoids the need to estimate the full conditional distribution of a multivariate response, relying instead on estimating only the conditional quantile of a univariate conformity score.

    \item We bound the conditional coverage error of the prediction sets generated by \RCP in terms of the accuracy of the conditional quantile estimate and the size of the calibration sample (see Section~\ref{sec:theory}).

    \item We evaluate our method on several benchmark datasets and compare it against state-of-the-art alternatives\footnote{The code to reproduce main experiments is available at \url{https://github.com/stat-ml/rcp}} (see Section~\ref{sec:experiments}). Our results demonstrate improved performance, particularly in terms of conditional coverage metrics such as worst slab coverage~\citep{romano2020classification} and conditional coverage error~\citep{dheur2024distribution}.
  \end{itemize}

\section{Background}
\label{sec:ccp}
  % !TEX root = ../main.tex
  
  Consider a regression problem that aims to estimate a $d$-dimensional response vector $y \in \mathcal{Y} = \mathbb{R}^d$ based on a feature vector $x \in \mathcal{X} \subseteq \R^p$ to predict. We denote by $\textup{P}_{X,Y}$ the joint distribution of $(X,Y)$ over $\mathcal{X} \times \mathcal{Y}$.

  Construction of prediction regions for regression problems is often based on distributional regression that focuses on fully characterizing the conditional distribution of a response variable given a covariate~\cite{klein2024distributional}. This approach improves uncertainty quantification and decision-making~\cite{Berger2019-ju}. From the conditional predictive distribution, prediction regions can be derived to capture values likely to occur with a given probability. However, these regions rely heavily on the predictive model's quality, and poorly estimated models can result in unreliable predictions. In the following, we present split-conformal prediction~\citep[SCP;][]{papadopoulos2002inductive}, a computationally efficient variant of the conformal prediction framework that allows generating reliable prediction regions, even when the predictive model is misspecified or inaccurate.

\paragraphformat{Split conformal prediction (SCP).}
  %Let $g\colon \mathcal{X} \to \mathcal{S}$ represent a regression model that maps a covariate $x \in \mathcal{X}$ to a prediction space $\mathcal{S}$. For example, if $g(x)$ estimates the conditional mean or quantile of $Y$ given $X = x$, then $\mathcal{S} = \mathcal{Y}$. If $g(x)$ instead estimates the conditional probability density function (PDF) of $Y$ given $X = x$, $\mathcal{S}$ corresponds to the space of PDFs over $\mathcal{Y}$.
  %
  Given a possibly misspecified predictive model $g(x)$, for any input $x \in \mathcal{X}$, SCP~\cite{papadopoulos2002inductive} generates a prediction set $\mathcal{C}_{\alpha}(x)$ at a user-specified confidence level $\alpha \in (0, 1)$ with \textit{marginal validity}~\cite{papadopoulos2008inductive}:
  \begin{equation}
  \label{eq:ge:coverage}
    \prob\bigl(Y \in \mathcal{C}_{\alpha}(X)\bigr) \ge 1-\alpha.
  \end{equation}
  To do so, SCP relies on a \textit{conformity score} function, $V\colon \XC \times \YC \to \R$, assigning larger value to worse agreement between $g(X)$ and $Y$. Let $\{(X_k,Y_k)\}_{k=1}^{\tcount}$ be a calibration set, with $\XC \subseteq \mathbb{R}^p$ and $\YC \subseteq \mathbb{R}^d$. SCP generates a prediction set $\mathcal{C}_{\alpha}(x)$ by computing an empirical quantile of the conformity scores $V(X_k, Y_k), k = 1, \dots, \tcount$:
  \begin{equation*}\label{eq:split_conformal}
    \resizebox{0.97\hsize}{!}{$
    \mathcal{C}_{\alpha}(x)
    = \ac{
      y\colon V(x, y) \le \q{1-\alpha}\pr{\sum_{k=1}^{\tcount} \frac{\delta_{V(X_k, Y_k)}}{\tcount+1} + \frac{\delta_{\infty}}{\tcount+1} }
    }
    $},
  \end{equation*}
  Here $\delta_v$ is the Dirac mass at $v$. For $p\in(0,1)$ and a
  distribution $P$ on $\R\cup\{+\infty\}$, we use the lower quantile
  $Q_p(P)=\inf\{t\in\R:P((-\infty,t])\ge p\}$, with $\inf\varnothing=+\infty$.
  The marginal guarantee follows from the exchangeability of calibration
  and test scores conditional on the fitted predictor.

\paragraphformat{Towards conditional validity of CP methods.}  
  Conditional validity requires that for $\textup{P}_X$-almost every $x$,
  \begin{equation}
  \label{eq:def:conditional-validity}
    \prob\pr{Y \in \mathcal{C}_{\alpha}(X) \mid X = x} \ge 1 - \alpha.
  \end{equation}
  Conditional coverage~\eqref{eq:def:conditional-validity} is stronger and implies marginal coverage~\eqref{eq:ge:coverage}. While classical conformal methods provide marginal validity~\eqref{eq:ge:coverage}, they do not ensure conditional validity. 

  Let us denote the conditional distribution \(\textup{P}_{\mathbf{V}|X=x}\) with $\mathbf{V}$ being a shorthand for $V(X, Y)$. The following oracle prediction set
  \begin{equation}
    \mathcal{C}_{\alpha}(x) = \bigl\{y \in \YC\colon V(x, y) \leq Q_{1-\alpha}(\textup{P}_{\mathbf{V}|X=x})\bigr\}
  \label{eq:oracle_cc}
  \end{equation}
  satisfies~\eqref{eq:def:conditional-validity} by the definition of the
  lower quantile. Achieving distribution-free conditional validity with
  informative prediction sets generally requires further assumptions
  on the data distribution~\cite{vovk2012conditional,lei2014distribution,foygel2021limits}.
  We next develop a method with finite-sample marginal validity and
  approximate conditional coverage.

\section{Rectified Conformal Prediction}
\label{sec:rcp}
  % !TEX root = ../main.tex

  The primary objective of our \textit{Rectified Conformal Prediction} (\RCP) method is to enhance the conditional coverage of any given conformity score while maintaining their exact marginal validity. Expression~\eqref{eq:oracle_cc} suggests that one could approximate the $(1-\alpha)$-quantile of the conditional distribution of the scores to construct the prediction set:
  \begin{equation*}
    \tilde{\mc{C}}_{\alpha}(x)
    =
    \ac{
      y \in \YC\colon V(x, y)\le \widehat{Q}_{1-\alpha}\pr{\textup{P}_{\mathbf{V}\mid X = x}}
    }.
  \end{equation*}
  This prediction set provides approximate conditional guarantees that depend on the accuracy of the quantile estimator. However, it fails to ensure exact marginal coverage which is an essential property for conformal prediction methods. 

\paragraphformat{A motivation for RCP.}
  Our \RCP method is designed to achieve exact conformal marginal validity and approximate conditional coverage. If the conditional score CDFs are continuous and share the same $(1-\alpha)$-quantile, thresholding at that population quantile gives both marginal and conditional coverage $1-\alpha$. \RCP estimates the corresponding score transformation and uses a split-conformal calibration threshold. The finite-sample bounds for that empirical threshold are given in \Cref{sec:theory}.

  Given a basic score $V(x,y)$, the oracle construction produces a
  rectified score $\tilde V_\star(x,y)$ satisfying, for
  $\textup{P}_X$-almost every $x$,
  \begin{equation}
    Q_{1-\alpha}\bigl(\textup{P}_{\tilde V_\star(X,Y)}\bigr)
    = Q_{1-\alpha}\bigl(\textup{P}_{\tilde V_\star(X,Y)\mid X=x}\bigr).
  \label{eq:mc_cc}
  \end{equation}
  Below we present two examples that show how one can construct the rectified scores satisfying~\eqref{eq:mc_cc}.

  \begin{example}
  \label{example:linear}
    Consider the rectified score $\tilde{V}(x, y) = V(x, y) / Q_{1-\alpha}(\textup{P}_{\mathbf{V}|X=x})$, with the assumption that $Q_{1-\alpha}(\textup{P}_{\mathbf{V}|X=x})> 0$ for any $x \in \mathcal{X}$. We can define the following prediction set, equivalent to~\eqref{eq:oracle_cc}:
    $\mathcal{C}_{\alpha}(x) = \{y \in \YC\colon \tilde{V}(x, y) \leq 1 \}$. 
    This prediction set has conditional coverage at least $1-\alpha$.
    \Cref{suppl:examplesA} proves \eqref{eq:mc_cc} for this score.
  \end{example}

  \begin{example}
  \label{example:additive} Consider the rectified score 
    $\tilde{V}(x, y) = V(x, y) - Q_{1-\alpha}(\textup{P}_{\mathbf{V}|X=x})$.
    The corresponding prediction set, also equivalent to~\eqref{eq:oracle_cc}, is:
    $\mathcal{C}_{\alpha}(x) = \{y \in \YC\colon \tilde{V}(x, y) \leq 0 \}$,
    and it has conditional coverage at least $1-\alpha$.
    \Cref{suppl:examplesB} proves \eqref{eq:mc_cc} for this score.
  \end{example}
  In the following, we generalize over these two basic examples and present a rich family of general score transformations that allow for score rectification.

\paragraphformat{\RCP with general transformations.}
  Recall that starting from a basic score function \(V(x, y)\), we develop a transformed score \(\tilde{V}(x, y)\) to achieve conditional validity at a given confidence level $\alpha$. To do so, we introduce a transformation to rectify the basic conformity score \(V\). 
  
  Consider a jointly measurable family $(t,v)\mapsto f_t(v)$ on
  $\mathbb T\times\R$, where $\mathbb T$ is a Borel subset of $\R$.
  Write $\tilde f_v(t)=f_t(v)$.
    
  \begin{assumption}
  \label{ass:tau}
    For any $t\in\mathbb T$, the function $v\in\R\mapsto \adj{t}(v)$ is strictly
    increasing and its image contains the range of $V$.
    There exists $\varphi\in\R$ such that $\adjinv$ is continuous and strictly
    increasing on $\mathbb T$, with image containing the range of $V$.
  \end{assumption}
  Under \Cref{ass:tau}, the inverses are defined on the range of $V$.
  For example, measurability of $f_t^{-1}$ as a function of $(t,z)$
  follows from $f_t^{-1}(z)\le v$ if and only if $z\le f_t(v)$.
  Set
  \begin{equation}
  \label{eq:def:V-varphi}
    V_\varphi(x,y) = \tilde{f}_\varphi^{-1}\bigl(V(x, y)\bigr)
  \end{equation}
  and denote $\mathbf{V} = V(X, Y)$, and $\mathbf{V}_\varphi = V_\varphi(X,Y)$. 
  We now define the following prediction set
  \begin{equation}
    \mc{C}_{\alpha}^*(x) = \bigl\{y \in \mc{Y}\colon V(x, y) \leq f_{\tau_\star(x)}(\varphi)\bigr\},
  \label{eq:rcp2}
  \end{equation}
  where 
  \begin{equation}
  \label{eq:definition-tau}
    \!\!\! \tau_\star(x) = \q{1 - \alpha}\big(\textup{P}_{\mathbf{V}_\varphi\mid X=x}\big)
    = \tilde{f}_\varphi^{-1}\bigl(\q{1 - \alpha}\bigl(\textup{P}_{\mathbf{V} \mid X = x}\bigr)\bigr),
  \end{equation}
  assuming $\tau_\star(x)\in\mathbb T$ for $\textup{P}_X$-almost every $x$.
  \Cref{suppl:cond-quantile} proves the second equality. The general
  construction gives \Cref{example:linear} with $f_t(v)=vt$,
  $\mathbb T=(0,\infty)$, and $\varphi=1$ when $V$ takes positive values.
  It gives \Cref{example:additive} with $f_t(v)=v+t$, $\mathbb T=\R$,
  and $\varphi=0$. In both cases $V_\varphi=V$.
  
  We may rewrite \eqref{eq:rcp2} in terms of the oracle rectified score
  $\tilde{V}_\star(x,y)= f_{\tau_\star(x)}^{-1}\bigl(V(x,y)\bigr)$: 
  \begin{equation}
  \label{eq:rcp2-rectified}
      \mc{C}_{\alpha}^*(x) = \bigl\{y \in \mc{Y}\colon \tilde{V}_\star(x, y) \leq \varphi\bigr\}.
  \end{equation}
  In \Cref{suppl:cond-quantile}, we establish that the rectified score satisfies~\eqref{eq:mc_cc}. With $\tilde{\mathbf{V}}_\star = \tilde{V}_\star(X,Y)$, for $\textup{P}_X$-almost every $x$,
  \begin{equation}
  \label{eq:key-relation}
    \varphi = Q_{1-\alpha}\bigl(\textup{P}_{\tilde{\mathbf{V}}_\star|X=x}\bigr) = Q_{1-\alpha}\bigl(\textup{P}_{\tilde{\mathbf{V}}_\star}\bigr).
  \end{equation}
  The common population threshold $\varphi$ gives conditional and marginal
  coverage at least $1-\alpha$. The conditional coverage equals $1-\alpha$
  wherever the conditional CDF is continuous at $\varphi$. In practice, we
  estimate $\tau_\star$ on held-out data.

\section{Implementation of RCP}
\label{sec:implementation_of_rcp}
\paragraphformat{The \RCP algorithm.}
  Rectified conformal prediction approach, as discussed above, requires a basic conformity score function \(V\), a transformation function \(\adj{t}\), and a calibration dataset of $N = \tcount + \ccount$ points. A critical step in the RCP algorithm is estimating the conditional quantile \(\widehat{\tau}(x) \approx \q{1-\alpha}\bigl(\textup{P}_{\mathbf{V}_\varphi \mid X = x} \bigr)\), which we discuss in detail below. \(\widehat{\tau}\) is learned on a separate part of calibration dataset composed of \(\ccount\) data points \( \{(X_k', Y_k') \colon k = 1, \dots , \ccount\} \). Subsequently, RCP uses SCP with the rectified scores $\tilde{V}(x, y) := \adj{\widehat{\tau}(x)}^{-1} \bigl(V(x, y)\bigr)$ instead of the basic scores $V(x, y)$. SCP is applied to the rectified scores computed on the second part of the calibration dataset: $\tilde{\mathbf{V}}_k = \tilde{V}(X_{k}, Y_{k}), k = 1, \dots, \tcount$.

  Finally, for a given test input \(x\) and miscoverage level \(\alpha\), RCP computes the prediction set as
  {\small
    \begin{equation}
    \label{eq:rcp_empirical}
      \mc{C}_{\alpha}(x) = \Bigl\{y \in \mc{Y}\colon \tilde{V}(x,y)  \leq Q_{1-\alpha}\Bigl(\sum_{k=1}^{\tcount} \frac{\delta_{\tilde{\mathbf{V}}_k}}{\tcount+1} + \frac{\delta_{\infty}}{\tcount+1} \Bigr)\Bigr\}.
    \end{equation}
  }
  The resulting \RCP procedure is summarized in \cref{algo:rcp}. We show exact marginal validity of \RCP and give a bound on its approximate conditional coverage in Section~\ref{sec:theory} below. 

  \begin{algorithm}[t!] 
    \caption{The RCP algorithm}
    \label{algo:rcp}
    \begin{algorithmic}[]
      \STATE \textbf{Input:}
      Calibration dataset $\mathcal{D}$, miscoverage level $\alpha$, score $V$,
      transformations $\{f_t\}$, reference threshold $\varphi$, test input $x$.

      % --- Calibration Stage ---
      \STATE $\triangleright$ \textbf{Calibration Stage}
      \STATE Split $\mathcal{D}$ into $\{(X_{k}, Y_{k})\}_{k = 1}^\tcount$ and $\{(X_{k}', Y_{k}')\}_{k = 1}^\ccount$.

      \STATE $\mathcal{D}_{\tau} \gets \{ (X_k', V_\varphi(X_k', Y_k')) \}_{k = 1}^{\ccount}$
      \STATE $\widehat{\tau} \gets$ estimate the conditional $(1-\alpha)$-quantile on $\mathcal D_\tau$.
      \FOR{$k = 1$ \textbf{to} $\tcount$}
        \STATE $\tilde{\mathbf{V}}_{k} \gets \adj{\widehat{\tau}(X_{k})}^{-1} \bigl(V(X_{k}, Y_{k})\bigr)$.
      \ENDFOR
      \STATE $k_{\alpha} \gets \lceil (1 - \alpha) (\tcount + 1) \rceil$.
      \STATE Append $+\infty$ to the calibration scores and sort, keeping ties.
      \STATE $Q\gets$ entry $k_\alpha$ of the sorted list.

      % --- Test Stage ---
      \STATE $\triangleright$ \textbf{Test Stage}
      \STATE \quad $\mathcal{C}_{\alpha}(x) \gets \bigr\{y \in \mathcal{Y}\colon \adj{\widehat{\tau}(x)}^{-1} \bigl(V(x, y)\bigr) \leq Q\bigl\}$.

      \STATE \textbf{Output:} $\mathcal{C}_{\alpha}(x)$.
    \end{algorithmic}
  \label{algo:RCP}
  \end{algorithm}

\paragraphformat{Estimation of $\tau_\star(x)$.}
  We present below several methods for estimating $\tau_\star(x)$. Interestingly, even coarse approximations of this conditional quantile can significantly improve conditional coverage; see the discussion in Section~\ref{sec:experiments}.

\paragraphformat{Quantile regression.}
\label{subsec:tau_qr}
  If $V_\varphi(X,Y)$ has a finite conditional first absolute moment at
  $x$, its lower $(1-\alpha)$-quantile minimizes the conditional pinball risk:
  \begin{equation}
  \label{eq:def:loss-x-unconditional}
    \tau_\star(x)\in\arg\min_{t\in\R}
      \E[\rho_{1-\alpha}(V_\varphi(X,Y)-t)\mid X=x],
  \end{equation}
  where $\rho_{1-\alpha}(u)=(1-\alpha)u_++\alpha(-u)_+$ is the
  pinball loss~\cite{koenker1978regression,koenker2001quantile}. The left
  and right derivatives in $t$ are the conditional CDF's left limit
  and value at $t$, each minus $1-\alpha$. The lower quantile satisfies
  the resulting minimizer condition, which allows ties and nonunique minimizers.
  In practice, the empirical quantile function \(\widehat{\tau}\) is obtained by minimizing the empirical pinball loss:
  \begin{equation}
  \label{eq:def:loss-pinball}
    \begin{aligned}
      R_m(\tau)
        &=\frac{1}{m}\sum_{k=1}^{m}
          \rho_{1-\alpha}\bigl(V_\varphi(X_k',Y_k')-\tau(X_k')\bigr),\\
      \widehat{\tau}
        &\in\arg\min_{\tau\in\mathcal{C}}\{R_m(\tau)+\lambda g(\tau)\},
    \end{aligned}
  \end{equation}
  where $g$ is a penalty function and $\mc{C}$ is a class of functions. When $\tau(x) = \theta^{\top} \Phi(x)$  where $\Phi$ is a feature map, and $g$ is convex, the optimization problem in~\eqref{eq:def:loss-pinball} becomes convex. Theoretical guarantees in this setting, are given, e.g., in~\cite{chen2005computational,koenker2005quantile}.

  Non-parametric methods have also been extensively explored, see, e.g.,~\cite{chernozhukov2005iv,chernozhukov2022fast}. \citet{takeuchi2006nonparametric} introduced the kernel quantile regression (KQR) framework, formulating quantile regression as minimizing the pinball loss in an RKHS with Tikhonov (squared-norm) regularization. It established some of the first theoretical guarantees for RKHS-based quantile models, deriving finite-sample generalization error bounds using Rademacher complexity; these results were later improved in~\cite{li2007quantile}

\paragraphformat{Local quantile regression.} 
\label{subsec:tau_local}
  The local quantile can be obtained by minimizing the empirical weighted expected value of the pinball loss function \(\rho_{1-\alpha}\), defined as follows:
  \begin{equation}
  \label{eq:def:reg-tau}
    \widehat{\tau}(x) \in \arg\underset{t\in\R}{\min} \ac{\sum_{k=1}^{\ccount} w_{k}(x) \rho_{1-\alpha}\pr{V_{\varphi}(X_{k}',Y_{k}')-t}},
  \end{equation}
  where the weights are nonnegative~\cite{bhattacharya1990kernel}. We
  use $w_k(x)=K_{h_X}(\|x-X_k'\|)/m$ with
  $K_h(r)=h^{-1}K_1(r/h)$. When the sum of weights is positive, the lower
  quantile of the weighted empirical CDF is a minimizer and belongs to
  $\mathbb T$ because it is an observed transformed score. If all weights
  vanish, any measurable fallback taking values in $\mathbb T$ minimizes
  \eqref{eq:def:reg-tau}.
  More recently, \citet{shen2024nonparametric} introduced a penalized non-parametric approach to estimating the quantile regression process (QRP) using deep neural networks with rectifier quadratic unit (ReQU) activations. \citet{shen2024nonparametric} derives upper bounds on the mean-squared error for quantile regression using deep ReQU networks, depending only on the approximation error and network. The bounds are shown to be tight for broad function classes (e.g., Hölder compositions, Besov spaces), implying that ReQU neural networks achieve minimax-optimal convergence rates for conditional quantile estimation. Notably, the theory requires minimal assumptions and holds even for heavy-tailed error distributions.

\section{Related Work}
\label{sec:cr-literature}
  % !TEX root = ../main.tex

  It is well known that obtaining exact conditional coverage for all possible inputs within the conformal prediction framework is impossible without making distributional assumptions~\cite{foygel2021limits}. However, the literature has proposed various relaxations of exact conditional coverage, focusing on different notions of approximate conditional coverage. 

  A first class of methods involves group-conditional guarantees~\cite{jung2022batch,ding2024class}, which provide coverage guarantees for a predefined set of groups. Another class partitions the covariate space into multiple regions and applies classical conformal prediction within each region~\cite{leroy2021md, alaa2023conformalized,kiyani2024conformal}. The significant limitation of these methods lies in the need to specify the groups or regions in advance.

  Other conformal methods aim to approximate conditional coverage by leveraging uncertainty estimates from the base predictor. When $d = 1$, Conformalized Quantile Regression~\citep[CQR;][]{romano2019conformalized} suggests constructing a conformalized prediction interval \(\mathcal{C}_{\alpha}(x)\) by leveraging two quantile estimates of \(Y \mid X = x\), denoted as \(\hat{q}_{\alpha/2}(x)\) and \(\hat{q}_{1-\alpha/2}(x)\). This approach yields prediction intervals that adapt to heteroscedasticity~\citep{kivaranovic2020adaptive}. By considering a version of CQR by~\citet{sesia2020comparison} whose conformity score is positive, we can draw a connection with \RCP. 
  The conformity score is
  \begin{equation*}
    V(x,y) = \abs{y - \mu_\alpha(x)} / \delta_\alpha(x), 
  \end{equation*}
  with $\mu_\alpha(x)= (\hat{q}_{1-\alpha/2}(x)+\hat{q}_{\alpha/2}(x))/2$ and $\delta_\alpha(x)=\hat{q}_{1-\alpha/2}(x) - \hat{q}_{\alpha/2}(x)$. Applying \RCP with $\adj{t}(v)=t v$ yields the following scaled transformed conformity scores:
  \begin{equation*}
    \tilde{V}(x,y) = \abs{y -  \mu_\alpha(x)} / \hat{\tau}(x),
  \end{equation*}
  where $\widehat{\tau}(x)$ is an estimator of the conditional $(1-\alpha)$-quantile of $\abs{Y -  \mu_\alpha(x)}$ given $X=x$. Thus, this particular variant of \RCP closely resembles the CQR approach but uses a different quantile estimate.

  In the context of multivariate prediction sets, given a predictor \(\mu(\cdot)\), a natural choice for the conformity score is \(V_\infty(x, y) = \|y - \mu(x)\|_{\infty}\), where $\| u \|_{\infty}= \max_{1 \leq t \leq d}(|u_i|)$ \citep{Diquigiovanni2021-bh}. This conformity score measures the prediction error associated with the predictor \(\mu\)~\cite{nouretdinov2001ridge,vovk2005algorithmic,vovk2009line}. Setting \(\adj{t}(v) = t v\) and $\varphi = 1$, the rectified conformity scores are given by $\tilde{V}(x, y) = V(x, y)/ \widehat{\tau}(x)$ where \(\widehat{\tau}(x) \approx \q{1-\alpha}\bigl(\textup{P}_{\mathbf{V}_\infty \mid X = x}\bigr)\), with $\mathbf{V}_\infty= V_\infty(X,Y)$. Thus, \RCP is similar to the approach proposed in~\cite{lei2018distribution}, but with a different choice of scaling function. 

  Methods utilizing conditional density estimation have been proposed to produce conformal prediction intervals that adapt to skewed data~\citep{sesia2021conformal}, to minimize the average volume~\citep[][denoted \texttt{DCP} in our paper]{Sadinle2016-yr} or to define more flexible highest-density regions~\citep{izbicki2022cd,plassier2024conditionally}. Probabilistic conformal prediction~\citep[PCP;][]{wang2023probabilistic} bypasses density estimation by constructing prediction sets as unions of balls centered on samples from a generative model. All these methods are either tailored to handle the scalar response ($d = 1$) or require an accurate conditional distribution estimate which might be hard to obtain in practical scenarios.

  \citet{guan2023localized} introduces a localized conformal prediction framework that adapts to data heterogeneity by weighting calibration points based on their similarity to the test sample. To do so, kernel-based localizers assign greater importance to nearby points, tailoring prediction intervals to local data patterns. \citet{amoukou2023adaptive} extend Guan's approach by replacing kernels with quantile regression forest estimators for improved performance. Although effective, these methods face challenges in high-dimensional or mixed-variable settings. 

  Several methods aim to transform conformity scores to improve approximate conditional coverage. 
  For example, \citet{johansson2021investigating}, following earlier works by~\citet{papadopoulos2011reliable,johansson2014regression,lei2018distribution}, investigate \emph{normalized conformity scores} (NCF), which enhance standard conformal prediction by adjusting prediction set according to instance difficulty. NCF can be represented within our framework through a specific choice of the function $f_{\tau}(v) = v/(\tau + \beta)$,  where $\beta$-values will put a greater emphasis on the difficulty estimation.  Notably, the estimation approach employed in these papers uses least-squares regression on residuals, in contrast to the quantile regression approach adopted in \RCP, which targets the oracle quantile equality~\eqref{eq:mc_cc}. \citet{han2022split} presents an approach that uses kernel density estimation to approximate the conditional distribution. Similarly, \citet{deutschmann2023adaptive} rescales the conformity scores based on an estimate of the local score distribution using the jackknife+ technique. However, these methods generally rely on estimating the conditional distribution of conformity scores, which is challenging in practice. \citet{dewolf2025conditional} studies conditional validity of normalized conformal predictors in oracle setting, i.e., when the optimal normalization is known.
  
  Recent work by~\citet{colombo2024normalizing} suggests to transform the conformity score employing a normalizing flow: $\tilde{V}(x,y)= b(V(x,y), x)$. The normalizing flow is trained to map the joint distribution $\textup{P}_{V,X}$ of the conformity score and attributes into a product distribution, $\textup{P}_{\tilde{\mathbf{V}}} \otimes \textup{P}_X$, where $\textup{P}_{\tilde{\mathbf{V}}}$ is an arbitrary univariate distribution. Notably, this condition is stricter than the common-quantile property~\eqref{eq:mc_cc}, as it enforces $\textup{P}_{\tilde{\mathbf{V}}|X=x} = \textup{P}_{\tilde{\mathbf{V}}}$ for almost every $x$ under $\textup{P}_X$. Consequently, learning such a transformation typically necessitates a larger sample size; see \Cref{sec:experiments}.

  One method~\cite{Xie2024BoostedCP} proposes to use a cross-validated boosting procedure to learn a new score function to be used in split-conformal prediction. The authors consider a specific family of possible score functions and corresponding loss functions tailored either to deviation from conditional coverage or interval length. This method has several limitations compared to our approach: limited set of score functions, tailored to one-dimensional targets, requires access to the train set, and high computation cost.

  \RCP method shares some similarities with that of~\citet{gibbs2025conformal}, which also performs a quantile regression of the conformity score with respect to the attribute $X$. There are two essential differences: firstly, \citet{gibbs2025conformal} work directly with the conformity score $V$, whereas we regress on a transformed score $V_\varphi$. Secondly, the manner in which the quantile regression result is used differs significantly. \RCP uses the quantile estimator to define the rectified score $\tilde{V}$, to which  the standard CP procedure, while \citet{gibbs2025conformal} propose a considerably more complex procedure; see \Cref{sec:experiments}.
  
  Finally, various recalibration methods have been proposed to improve marginal coverage~\citep{dheur2023large} or conditional coverage~\citep{dey2022conditionally}. While these methods can also be interpreted within the conformal prediction framework~\citep{marx2022modular,dheur2024probabilistic}, they often require modifications to the training procedure, making them less broadly applicable than purely conformal methods.

\section{Theoretical Guarantees}
\label{sec:theory}
  % !TEX root = ../main.tex

  We study the marginal and conditional validity of the predictive set
  $\mc{C}_{\alpha}(x)$ defined in~\eqref{eq:rcp_empirical}. We use $n$
  for the final calibration size and $m$ for the size of the sample
  used to estimate the conditional quantile. Let $\mathcal{D}_{\mathrm{fit}}$
  denote the data used to fit $V$ and $\widehat{\tau}$.
  \begin{assumption}\label{ass:tau-in-T}
    For any $x\in\XC$, we have $\widehat{\tau}(x)\in\mathbb{T}$.
  \end{assumption}
  The first result is the standard split-conformal guarantee.

  \begin{theorem}\label{thm:coverage:marginal}
    Assume \Cref{ass:tau} and \Cref{ass:tau-in-T}, and suppose the $n$
    calibration scores and the test score are exchangeable conditional
    on $\mathcal{D}_{\mathrm{fit}}$. For any $\alpha\in(0,1)$, set
    $k_\alpha=\lceil(n+1)(1-\alpha)\rceil$. Then, almost surely,
    \begin{multline*}
      \prob\pr{Y_{n+1}\in\mathcal{C}_\alpha(X_{n+1})\mid\mathcal{D}_{\mathrm{fit}}}\\
      \ge \frac{k_\alpha}{n+1}\ge 1-\alpha .
    \end{multline*}
    If the scores are almost surely distinct, the coverage equals
    $k_\alpha/(n+1)<1-\alpha+1/(n+1)$.
  \end{theorem}
  The proof is given in \Cref{suppl:marginal}.

  Write $c_\alpha(x)$ for the coverage conditional on
  $\mathcal{D}_{\mathrm{fit}}$ and $X=x$, averaged over the final calibration
  sample. Given $\mathcal{D}_{\mathrm{fit}}$, let $F$ and $F_x$ denote the
  marginal and $X=x$ conditional CDFs of the rectified score $\tilde V(X,Y)$.
  We measure the quantile estimation error on the coverage scale by
  \begin{equation}\label{eq:def:epsilon-tau}
    \epsilon_\tau(x)
    = \prob\pr{V_\varphi(X,Y)\le \tau(x)\mid X=x,\mathcal{D}_{\mathrm{fit}}}
      -(1-\alpha).
  \end{equation}

  \begin{theorem}\label{thm:coverage:conditional}
    Assume \Cref{ass:tau} and \Cref{ass:tau-in-T}, with i.i.d. final
    calibration and test pairs independent of $\mathcal{D}_{\mathrm{fit}}$.
    Suppose $F$ is continuous and
    $F_x\circ F^{-1}$ is $\mathrm{L}$-Lipschitz on $(0,1)$ for
    $\mathrm{P}_X$-almost every $x$. For $\alpha\in[1/(n+1),1)$, set
    $\beta=F(\varphi)$, $k_\alpha=\lceil(n+1)(1-\alpha)\rceil$, and
    $\mu=k_\alpha/(n+1)$. Then for $\mathrm{P}_X$-almost every $x$,
    \begin{multline}\label{eq:def:conditional-bound}
      \big|c_\alpha(x)-(1-\alpha)-\epsilon_{\widehat{\tau}}(x)\big|\\
      \le \mathrm{L}\left(|\beta-\mu|+\frac{1}{4\sqrt{n+2}}\right).
    \end{multline}
    For $0<\alpha<1/(n+1)$, the threshold is $+\infty$ and coverage is one.
  \end{theorem}

  The proof and sharper one-sided bounds expressed through regularized
  incomplete beta functions are given in \Cref{suppl:conditional:validity}.
  Since $\beta=1-\alpha+\E_X[\epsilon_{\widehat{\tau}}(X)]$ and
  $0\le \mu-(1-\alpha)<1/(n+1)$, the bound separates the error at $x$,
  the average quantile error, and a calibration term of order $n^{-1/2}$
  for fixed $\mathrm{L}$. This order can occur even with the oracle quantile
  (see \Cref{suppl:conditional:validity}).

  Let $\mathcal{L}_x(t)$ be the conditional pinball risk
  \begin{equation}\label{eq:def:loss-x}
    \mathcal{L}_x(t)=\E[\rho_{1-\alpha}(V_\varphi(X,Y)-t)
                                  \mid X=x,\mathcal{D}_{\mathrm{fit}}].
  \end{equation}
  If $V_\varphi(X,Y)$ has a finite conditional first moment and its CDF
  given $X=x$ is $D_x$-Lipschitz, then \Cref{eq:quantile-conditional-pinball-loss}
  gives
  \[
    |\epsilon_\tau(x)|
    \le \sqrt{2D_x\big(\mathcal{L}_x(\tau(x))-
                              \mathcal{L}_x(\tau_\star(x))\big)}.
  \]

\paragraphformat{Local quantile regression.}
  We now control $\epsilon_{\widehat{\tau}}(x)$ for the local estimator in
  \eqref{eq:def:reg-tau}, using kernel weights proportional to
  $K_{h_X}(\|X_i'-x\|)$. Write
  $H_x(v)=\prob(V_\varphi(X,Y)\le v\mid X=x)$ with the fitted score held fixed.
  \begin{assumption}\label{ass:kernel-cdf}
    For every $v\in\R$, the map $x\mapsto H_x(v)$ is $\mathrm{M}$-Lipschitz.
    The kernel $K_h$ is nonnegative, nonincreasing on $\R_+$, and bounded,
    with $0<K_h(0)<\infty$.
  \end{assumption}

  For fixed $x$ and $h=h_X$, let $W_h=K_h(\|X-x\|)$ and
  $C_h(x)=\E W_h$. When $C_h(x)>0$, set
  \[
    b_h(x)=\frac{\E[W_h\min\{1,\mathrm{M}\|X-x\|\}]}{C_h(x)}.
  \]
  This bounds the bias of the population weighted CDF.

  \begin{proposition}\label{cor:epsilon-tau:local-cdf}
    Assume \Cref{ass:kernel-cdf}, $C_h(x)>0$, and that $H_x$ is continuous.
    Let the $m$ quantile fitting observations be i.i.d. copies of $(X,Y)$,
    independent of the data used to fit $V$. For any $\gamma\in(0,1)$,
    with probability at least $1-\gamma$ over these observations,
    \begin{multline*}
      |\epsilon_{\widehat{\tau}}(x)|\le b_h(x)\\
      +\frac{1}{C_h(x)}
        \left(\sqrt{\frac{2\E[W_h^2]\log(2/\gamma)}{m}}
               +\frac{2K_h(0)\log(2/\gamma)}{3m}\right).
    \end{multline*}
  \end{proposition}
  The proof is given in \Cref{suppl:cdf}. Under the conditions in
  \Cref{suppl:asympt}, local quantile regression gives conditional coverage
  converging in probability to $1-\alpha$ at a fresh $X$. This holds, in
  particular, when $n,m\to\infty$, $h_X\to 0$, and $m h_X^{d_X}\to\infty$,
  where $d_X$ is the covariate dimension.

\section{Experiments}
\label{sec:experiments}
  % !TEX root = ../main.tex

  \begin{figure}[t!]
    \centering
    \includegraphics[width=\linewidth]{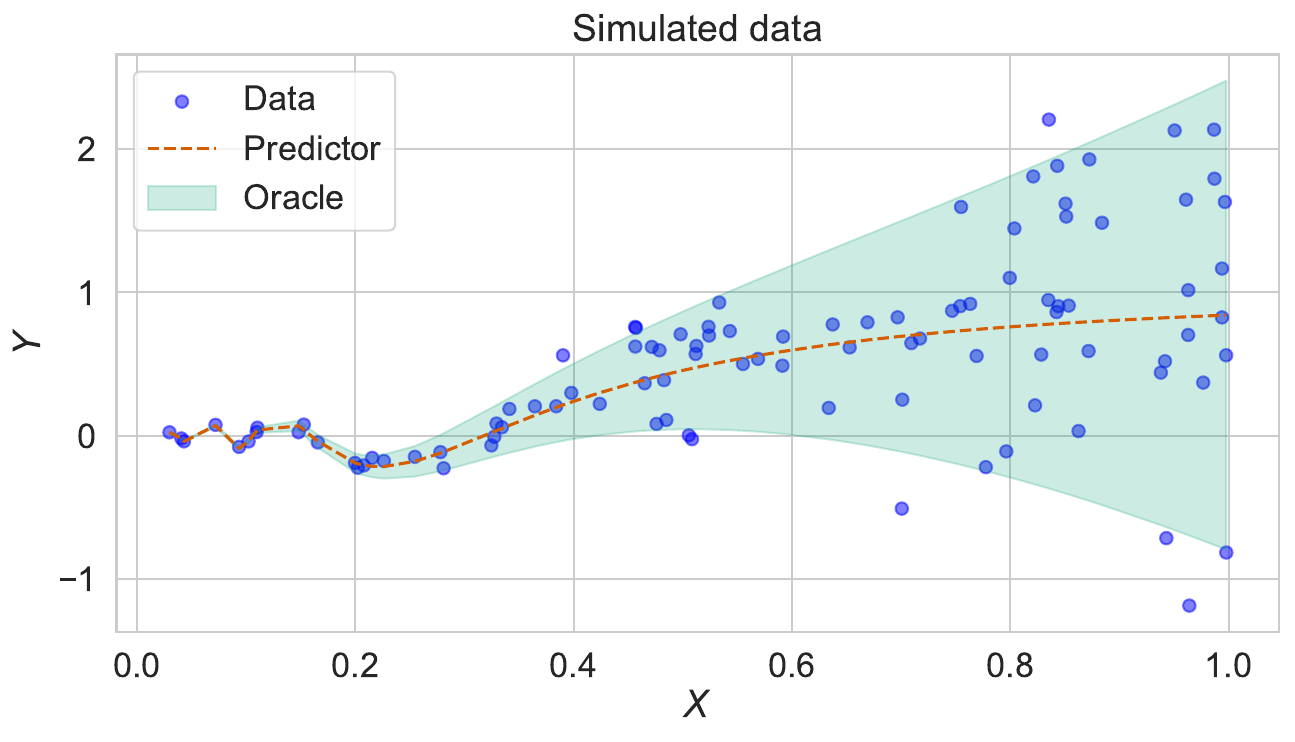}
    \caption{Oracle data distribution, sample data and predictor for the toy dataset.}
    \label{fig:toy}
  \end{figure}

\subsection{Toy example}
  Let us consider the following data-generating process:
  \begin{equation*}
    X\sim \mathrm{Beta}(1.2, 0.8), \quad Y \mid X=x\sim \gauss(\mu(x), x^4).
  \end{equation*}
  where $\mu(x) = x \sin(x) $. Figure~\ref{fig:toy} shows a realization with $n = 100$ data points. Our goal is to investigate the influence of the quality of the \((1-\alpha)\)-quantile estimate \(\widehat{\tau}\) on performance.
  
  We set $\alpha=0.1$ and consider the conformity score \(V(x,y) = |y - \mu(x)|\). In this case, the \((1-\alpha)\)-quantile of \(V(x,Y) \mid X=x\) is known and we denote it by \(\q{1-\alpha}(x)\). Given \(\omega \in [0,1]\), we set \(\widehat{\tau}(x) \sim (1-\omega) \q{1-\alpha}(x) + \omega \epsilon(x)\), where we consider \(\epsilon(x) \sim \gauss(0, x^4)\). We perform $1000$ experiments and report the $10\%$ lower value of $x \in [0,1] \mapsto \prob(Y_{\tcount+1}\in\mathcal{C}_{\alpha}(x) \mid X=x)$; the results can be found in~\Cref{table:toy}.
  For \(\omega=0\), the conditional quantile error vanishes. The final
  split-conformal threshold is still estimated from a finite sample.
  Marginal validity holds at every contamination level. The reported
  lower-tail conditional coverage decreases as the quantile estimate deteriorates.

  \begin{table}[t]
    \centering
    \begin{tabular}{lccccc}
      \toprule
      $\omega$ & $0$ & $1/3$ & $2/3$ & $1$ \\
      \midrule
      \textsc{Coverage} & $90 \pm 01$ & $84 \pm 01$ & $75 \pm 03$ & $59 \pm 07$ \\
      \bottomrule
    \end{tabular}
    \caption{Local coverage on the adversarially selected 10\% of the data, $\omega$ corresponds to the level of contamination of the score quantile estimate.}
    \label{table:toy}
  \end{table}

\subsection{Real-world experiment}
  We use publicly available regression datasets which are also considered in~\citep{Tsoumakas2011-wf,Feldman2023-cc,wang2023probabilistic} and only keep datasets with at least 2 outputs and 2000 total instances. The characteristics of the datasets are summarized in \cref{suppl:details}.

  \begin{figure*}[t!]
    \centering
    \includegraphics[width=\linewidth]{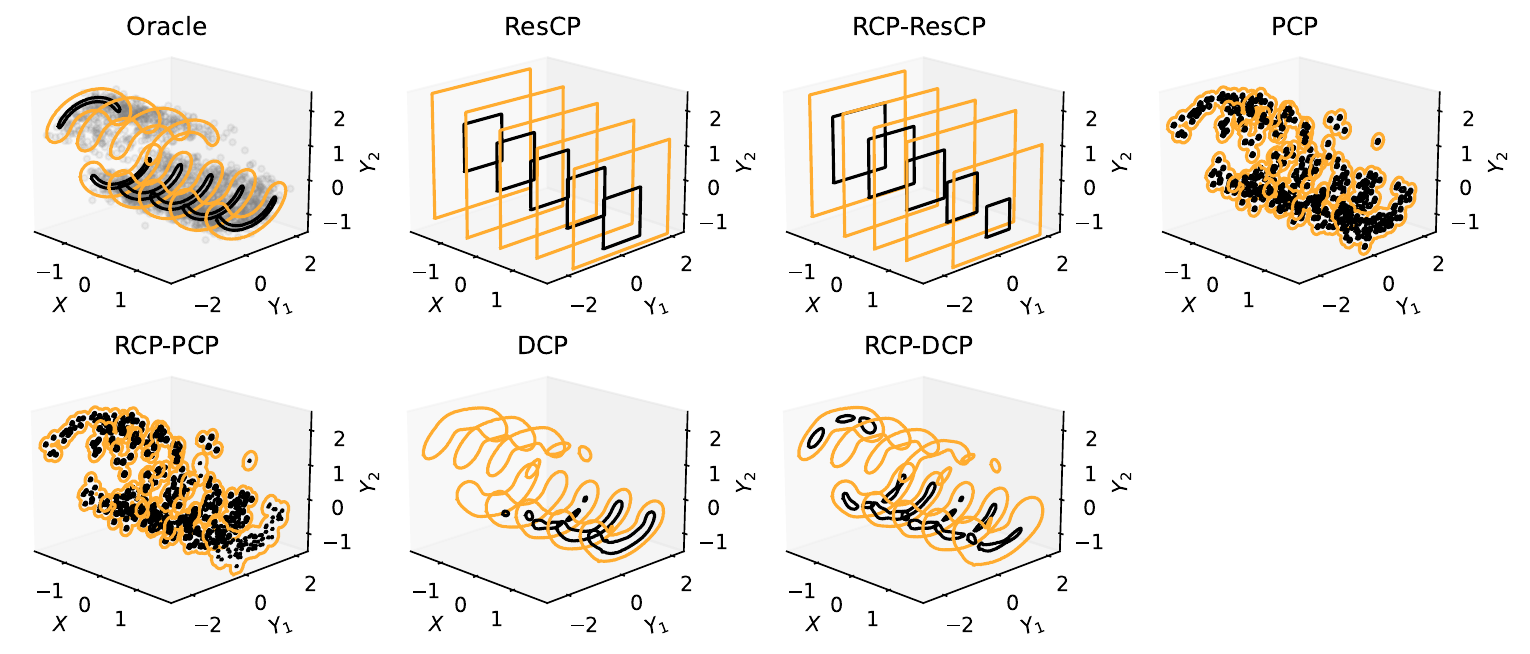}
    \caption{Examples of prediction sets on synthetic dataset where the output has a bivariate and bimodal distribution.}
    \label{fig:contours/Mixture_two_moons_heteroskedastic}
  \end{figure*}

\paragraphformat{Base predictors.}
  We consider two base predictors, both parameterized by a fully connected neural network with three layers of 100 units and ReLU activations.

  The \textit{mean predictor} estimates the mean $\hat{\mu}_i(x)$ of the distribution for each dimension $i \in [d]$ given $x \in \mathcal{X}$. Since it only provides a point estimate, it does not capture uncertainty.

  The \textit{mixture predictor} models a mixture of $K$ Gaussians, enabling it to represent multimodal distributions.
  Given \(x \in \mathcal{X}\), the model outputs \(z(x) \in \mathbb{R}^K\) (logits for mixture weights), \(\mu(x) \in \mathbb{R}^{K \times d}\) (mean vectors), and \(L(x) \in \mathbb{R}^{K \times d \times d}\) (lower triangular Cholesky factors). The mixture weights \(\pi(x) \in \mathbb{R}^K\) are obtained by applying the softmax function to \(z(x)\), and the covariance matrices \(\Sigma(x) \in \mathbb{R}^{K \times d \times d}\) are computed as \(\Sigma_k(x) = L_k(x) L_k(x)^\top\).
  The conditional density at \(y \in \mathcal{Y}\), given \(x \in \mathcal{X}\), is:
  \[
    \hat{p}(y \mid x) = \sum_{k=1}^K \pi_k(x) \cdot \mc{N}\bigl(y \mid \mu_k(x), \Sigma_k(x)\bigr),
  \]
  where \(\mc{N}\bigl(y \mid \mu_k(x), \Sigma_k(x)\bigr)\) is a Gaussian density with mean \(\mu_k(x)\) and covariance matrix \(\Sigma_k(x)\).

\paragraphformat{Methods.}
  We compare \RCP with four split-conformal prediction methods from the literature: \texttt{ResCP} \citep{Diquigiovanni2021-bh}, \texttt{PCP} \citep{wang2023probabilistic}, \texttt{DCP} \citep{Sadinle2016-yr}, and \texttt{SLCP} \citep{han2022split}. \texttt{ResCP} uses residuals as conformity scores. To handle multi-dimensional outputs, we follow~\citep{Diquigiovanni2021-bh} and define the conformity score as the $l^\infty$ norm of the residuals across dimensions, i.e., \( V(x, y) = \max_{i \in [d]} |\hat{\mu}_i(x) - y_i| \). \texttt{PCP} constructs the prediction set as a union of balls, while \texttt{DCP} defines the prediction set by thresholding the density.  \texttt{ResCP} is compatible with the \textit{mean predictor}, whereas \texttt{PCP} and \texttt{DCP} are compatible with the \textit{mixture predictor}. Finally, \texttt{SLCP}, like \RCP, is compatible with any conformity score and base predictor. For \RCP, we compute an estimate $\widehat{\tau}(x)$ (see \cref{subsec:tau_qr}) using quantile regression with a fully connected neural network composed of 3 layers with 100 units.

\paragraphformat{Visualization on a synthetic dataset.}
  \cref{fig:contours/Mixture_two_moons_heteroskedastic} illustrates example prediction sets for different methods. The orange and black contour lines represent confidence levels of \(\alpha = 0.1\) and \(\alpha = 0.8\), respectively. The first panel shows the highest density regions of the oracle distribution, while the subsequent panels display prediction regions obtained by different methods, both before and after applying \RCP. We can see that combining \RCP with \texttt{ResCP}, \texttt{PCP}, or \texttt{DCP} results in prediction sets that more closely align with those of the oracle distribution.

\paragraphformat{Experimental setup.}
  We reserve 2048 points for calibration. The remaining data is split between 70\% for training and 30\% for testing. The base predictor is trained on the training set, while the baseline conformal methods use the full calibration set to construct prediction sets for the test points. In \RCP, the calibration set is further divided into two parts: one for estimating $\widehat{\tau}(x)$ and the other as the proper calibration set for obtaining intervals. This ensures that all methods use the same number of points for uncertainty estimation. When not specified, we used the adjustment $f_t(v) = t + v$. Additional details on implementation and hyperparameter tuning are provided in \cref{suppl:details}.

  \begin{figure*}[t!]
    \centering
    \includegraphics[width=0.9\linewidth]{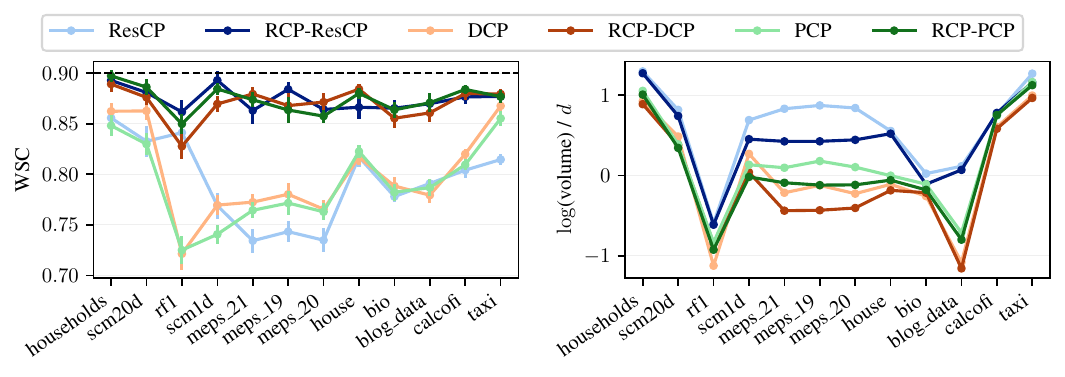}
    \caption{Worst-slab coverage and volume for three conformal methods and their RCP counterparts, on datasets sorted by total size.}
    \label{fig:pointplot/rcp_vs_base/wsc}
  \end{figure*}

\paragraphformat{Evaluation metrics.}
  To evaluate conditional coverage, we use \textit{worst-slab coverage}~(WSC, \citealp{cauchois2020knowing}; \mbox{\citealp{romano2020classification}}) with $\delta = 0.2$ and the \textit{conditional coverage error}, computed over a partition of \(\mathcal{X}\), following~\citet{Dheur2025-br}.
  To evaluate sharpness, we also report the median of the logarithm of the prediction set volume, scaled by the dimension $d$.

\paragraphformat{Main results.}
  \cref{fig:pointplot/rcp_vs_base/wsc} presents the worst-slab coverage and volume for different conformity scores, both with and without \RCP. Similarly, \cref{fig:pointplot/rcp_vs_slcp/wsc} compares worst-slab coverage between \texttt{SLCP} and \RCP. Additional results, including conditional coverage error and marginal coverage, are provided in \cref{sec:additional_results}.

  \begin{figure}[t!]
    \centering
    \includegraphics[width=\linewidth]{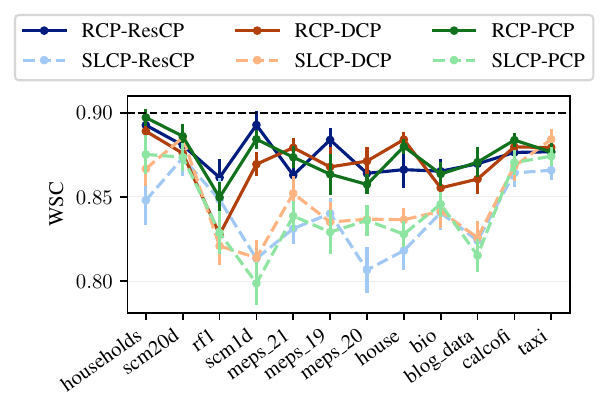}
    \caption{Worst-slab coverage for \RCP and \texttt{SLCP} in combination with different conformity scores, on datasets sorted by total size.}
    \label{fig:pointplot/rcp_vs_slcp/wsc}
  \end{figure}

  In the left panel of \cref{fig:pointplot/rcp_vs_base/wsc}, we observe that \texttt{ResCP}, \texttt{PCP}, and \texttt{DCP} fail to reach the nominal level of conditional coverage for most datasets. In contrast, all variants of \RCP significantly improve coverage across all datasets. Similarly, \cref{fig:pointplot/rcp_vs_slcp/wsc} shows that \RCP often achieves better conditional coverage than \texttt{SLCP}, particularly on larger datasets. \cref{fig:pointplot/rcp_vs_base/horizontal} in \cref{sec:additional_results} confirms these findings with the conditional coverage error. Finally, as expected, all methods achieve marginal coverage.

  In the right panel of \cref{fig:pointplot/rcp_vs_base/wsc}, we observe that \RCP improves the median prediction set volume compared to non-\RCP variants in addition of improving conditional coverage.

  Finally, \cref{sec:volumes} compares average and median volumes of prediction sets produced by direct conformal methods and their \RCP counterparts. Direct methods obtain a smaller average volume while \RCP obtains a smaller median volume. %The larger average volume of \RCP is explained by the larger regions produced for instances with larger uncertainty. When comparing the median volume, outliers are ignored.

\paragraphformat{Additional experiments.}
  We complement main results with multiple experiments aiming at studying variations of the proposed method and comparing it with some additional competitors.

  \cref{sec:tau_testimation} discusses the estimation of $\widehat{\tau}(x)$ using either a neural network or local quantile regression for which we have bounds the conditional coverage. On most datasets, the neural network slightly outperforms local quantile regression, which is expected due to its flexibility.
  \cref{sec:adj_function,sec:additional_adjustment} discuss the choice of adjustment function. For certain adjustment functions, the domain of the scores $v = V(x, y)$ must be restricted to a subset of $\R$ to satisfy \Cref{ass:tau-in-T}. Notably, $f_t(v) = tv$ requires $v > 0$, $f_t(v) = \exp(tv)$ requires require $v>1$.
  
  \cref{sec:comparison_cqr} directly compares the proposed method with CQR, showing that CQR already obtains a competitive conditional coverage but is outperformed by RCP-DCP in average volume. \cref{sec:simple_baseline} presents an additional study comparing RCP with CP and CQR methods, that we adapted to multidimensional target setting. For these experiments, the model predicts parameters of a multivariate normal distribution and we use the score based on the corresponding Mahalanobis distance. We demonstrate that \RCP improves conditional coverage over classic CP and also benefits from the custom score to outperform CQR.

  \cref{sec:CV_experiment} considers an approach to improve data efficiency. Instead of dividing the calibration dataset $\mathcal{D}$ into two parts to estimate $\widehat{\tau}$, we compute out-of-sample conformity scores on the training dataset $\mathcal{D}_{\text{train}}$ using $K$-fold cross-validation. This results in improved conditional coverage at the cost of training $K$ additional models.

  \cref{sec:comparison_cpcg} provides an additional comparison with Conditional Prediction with Conditional Guarantees (CPCG; \citet{gibbs2025conformal}). CPCG obtains a competitive conditional coverage but is 200-100000 times slower than \RCP overall, limiting its applicability.

\section{Conclusion}
\label{sec:conclusion}
  % !TEX root = ../main.tex

We present a new approach to improve the conditional coverage of the conformal prediction set while preserving finite-sample marginal validity.
Our method constructs prediction sets by adjusting the conformity scores using an appropriately defined conditional quantile, allowing \RCP\ to automatically adapt the prediction sets against heteroscedasticity.
Our theory gives finite-sample bounds on conditional coverage in terms of the quantile estimation error and the calibration sample size. It also establishes asymptotic conditional validity for local quantile regression under the stated regularity and bandwidth conditions.

\section*{Acknowledgments}
%\label{sec:acknowledgments}
V.P. carried out this study during his PhD under the supervision of E.M., and subsequently as a research engineer at the Centre Lagrange in Paris. Part of E.M.'s work was also conducted under the auspices of the Centre Lagrange. E.M. was all partially funded by the European Union (ERC-2022-SyG, 101071601). Views and opinions expressed are however those of the author(s) only and do not necessarily reflect those of the European Union or the European Research Council Executive Agency. Neither the European Union nor the granting authority can be held responsible for them.

\section*{Impact Statement}
  % !TEX root = ../main.tex

This work contributes to the broader effort to improve the interpretability and statistical reliability of machine learning algorithms. Prediction intervals with exact marginal and approximate conditional coverage offer a useful tool for conveying uncertainty regarding the accuracy of ML models, which is essential for increasing their fairness, and fostering wider acceptance.

%This paper presents work whose goal is to advance the field of Machine Learning. There are many potential societal consequences of our work, none which we feel must be specifically highlighted here.

\bibliography{conformal}
\bibliographystyle{icml2025}

%%%%%%%%%%%%%%%%%%%%%%%%%%%%%%%%%%%%%%%%%%%%%%%%%%%%%%%%%%%%

\newpage
\appendix
\onecolumn

\section{Additional Experiments}
\label{suppl:additional_experiements}
  % !TEX root = ../main.tex

\subsection{Additional results on marginal coverage and conditional coverage error}
\label{sec:additional_results}
  \cref{fig:pointplot/rcp_vs_base/horizontal} extends the results of \cref{fig:pointplot/rcp_vs_base/wsc} by displaying additionally the marginal coverage and conditional coverage error. As expected, all methods obtain a correct marginal coverage. Furthermore, the methods with the best worst slab coverage (closest to $1 - \alpha$) also obtain a small conditional coverage error, supporting our conclusions in \cref{sec:experiments}.

  \begin{figure}[H]
    \centering
    \includegraphics[width=0.8\linewidth]{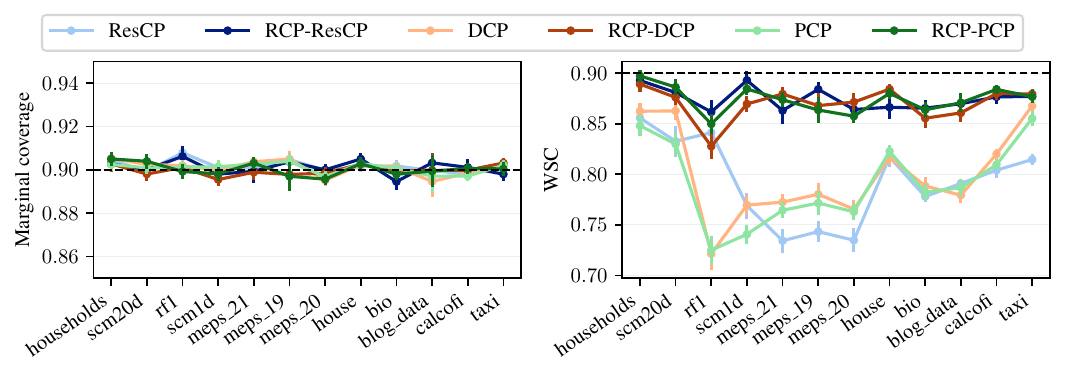}
    \caption{Marginal coverage and conditional coverage error for three conformal methods and their \RCP counterparts, on datasets sorted by total size.}
    \label{fig:pointplot/rcp_vs_base/horizontal}
  \end{figure}

\subsection{Estimation of conditional quantile function}
\label{sec:tau_testimation}
  \cref{fig:pointplot/kernel_0.01/horizontal} compares two ways of estimating $\widehat{\tau}$ (see \cref{subsec:tau_qr}). RCP$_\text{MLP}$ corresponds to quantile regression based on a neural network as in \cref{sec:experiments}, while RCP$_\text{local}$ corresponds to local quantile regression. On many datasets, the more flexible RCP$_\text{MLP}$ is able to obtain better conditional coverage. However, local quantile regression has theoretical guarantees on its conditional coverage (see \cref{sec:theory}).

  \begin{figure}[H]
    \centering
    \includegraphics[width=0.8\linewidth]{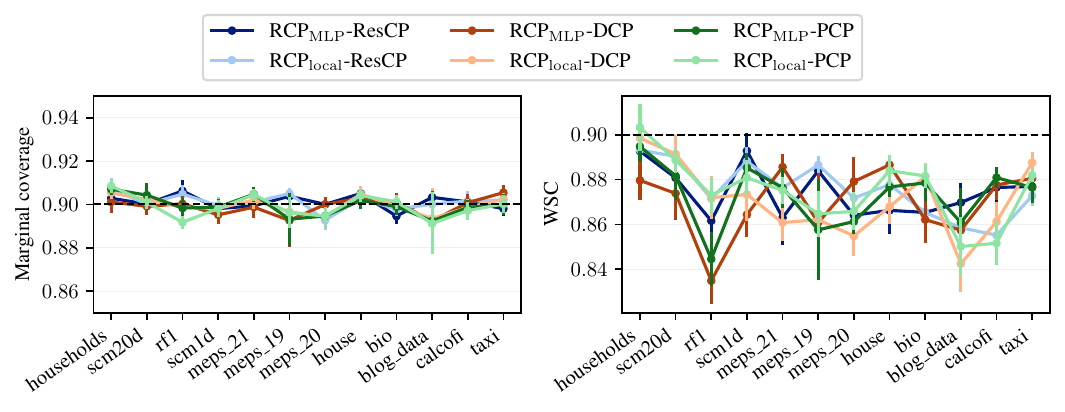}
    \caption{Marginal coverage and conditional coverage error for two types of quantile estimators in combination with different conformal methods, on datasets sorted by total size.}
    \label{fig:pointplot/kernel_0.01/horizontal}
  \end{figure}

\newpage

\subsection{Choice of adjustment function}
\label{sec:adj_function}
  \Cref{fig:pointplot/adjustment_issue/horizontal} compares \RCP with difference ($-$) and linear ($*$) adjustments when combined with the DCP method. Since \RCP with any adjustment function adheres to the SCP framework, marginal coverage is guaranteed, as shown in Panel 1.

  The conformity score for DCP is defined as \( V(x, y) = -\log \hat{p}(y \mid x) \), which can take negative values, implying that \( \mathbb{T} = \mathbb{R} \). However, the linear adjustment requires \( \mathbb{T} \subseteq \mathbb{R}_+^* \), violating \cref{ass:tau} and resulting in a failure to approximate conditional coverage accurately. This issue is evident in Panel 2. In contrast, the difference adjustment does not impose such a restriction.

  Panel 3 compares PCP and ResCP when used with difference and linear adjustments. Since the conformity scores for these methods are always positive, i.e., \( \mathbb{T} = \mathbb{R}^*_+ \), both adjustment methods satisfy \cref{ass:tau}. In general, we observe no significant differences between the two adjustment methods.

  \begin{figure}[H]
    \centering
    \includegraphics[width=0.66\linewidth]{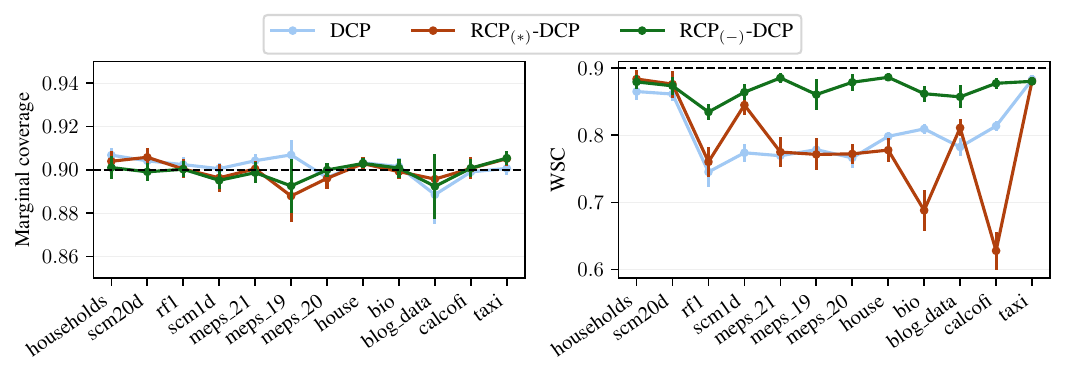}
    \includegraphics[width=0.32\linewidth]{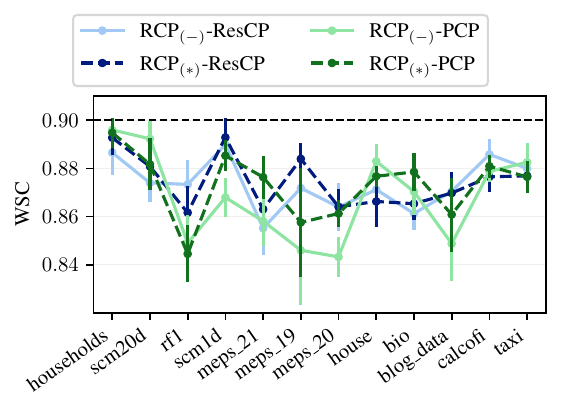}
    \caption{Marginal coverage and conditional coverage error obtained for two types of adjustments.}
    \label{fig:pointplot/adjustment_issue/horizontal}
  \end{figure}

\subsection{Additional adjustment functions}
\label{sec:additional_adjustment}
  We consider two additional adjustments functions, namely $f_t(v) = \exp{(t + v)}$, denoted $\exp -$, and $f_t(v) = \exp{(t v)}$, denoted $\exp *$. To apply these custom adjustment functions we need to ensure that the conditions \Cref{ass:tau} and \Cref{ass:tau-in-T} are satisfied.
  %\todo[inline]{Does $\exp *$ require scores greater than 1?}
  For the first function we have: $\tilde{f}_{\varphi}^{-1}(v) = (\ln v) - \varphi \in \mathbb{T}$ and $\varphi = 0$. Then $\tilde{f}_{\varphi}^{-1}(v) > 0 \; \Rightarrow \ln v > 0 \; \Rightarrow v > 1$. For the second function we can take $\varphi = 1$ and by similar argument we arrive at the same requirement $v > 1$. In practice, conformity scores are usually
   non-negative as is the case with PCP and residual scores that we consider here, and we can always add a constant $1$ to satisfy this requirement.
  
  \cref{fig:pointplot/adjustments_pcp/horizontal,fig:pointplot/adjustments_res/horizontal} show the marginal coverage and conditional coverage error obtained with these adjustment functions.

  \begin{figure}[H]
    \centering
    \includegraphics[width=0.8\linewidth]{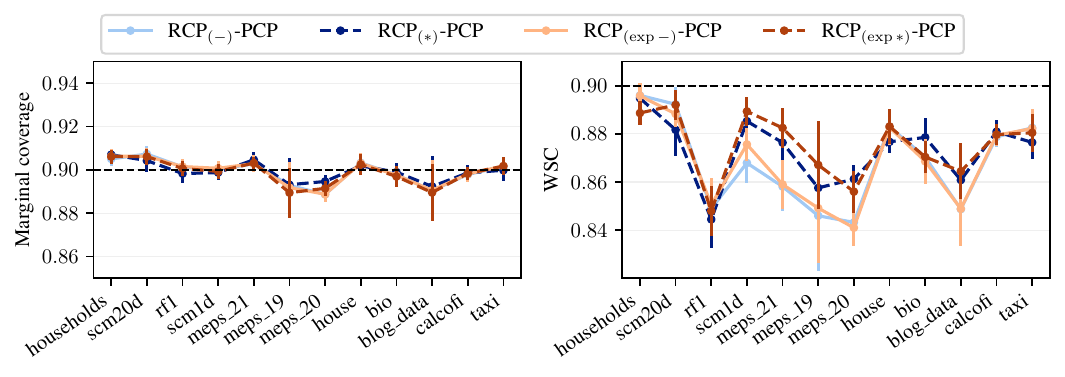}
    \caption{Marginal coverage and conditional coverage error for two additional types of adjustments combined with the method PCP.}
    \label{fig:pointplot/adjustments_pcp/horizontal}
  \end{figure}

  \begin{figure}[H]
    \centering
    \includegraphics[width=0.8\linewidth]{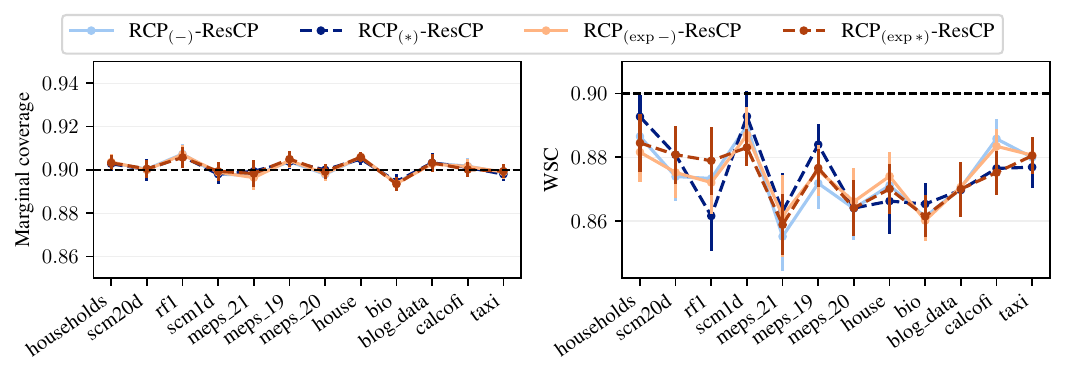}
    \caption{Marginal coverage and conditional coverage error for two additional types of adjustments combined with the method ResCP.}
    \label{fig:pointplot/adjustments_res/horizontal}
  \end{figure}

\subsection{Direct comparison with CQR}
\label{sec:comparison_cqr}
  Here we present a direct comparison of \RCP with Conformalized Quantile Regression (CQR; \citet{romano2019conformalized}). We use the same underlying neural network architectures for the models as in our main experiment. Similarly to ResCP, to handle multi-dimensional outputs, we follow~\citep{Diquigiovanni2021-bh} and define the conformity score of CQR as the $l^\infty$ norm of the CQR conformity scores across dimensions. Specifically, we compare CQR to DCP and its RCP-DCP counterpart, which achieves the best median volume overall.

  \cref{fig:rcp_vs_base_cqr/cond_coverage_and_size} shows that CQR matches the conditional coverage of RCP-DCP. However, it produces larger median prediction sets due to less flexible shapes.

  \begin{figure}[H]
    \centering
    \includegraphics[width=0.8\linewidth]{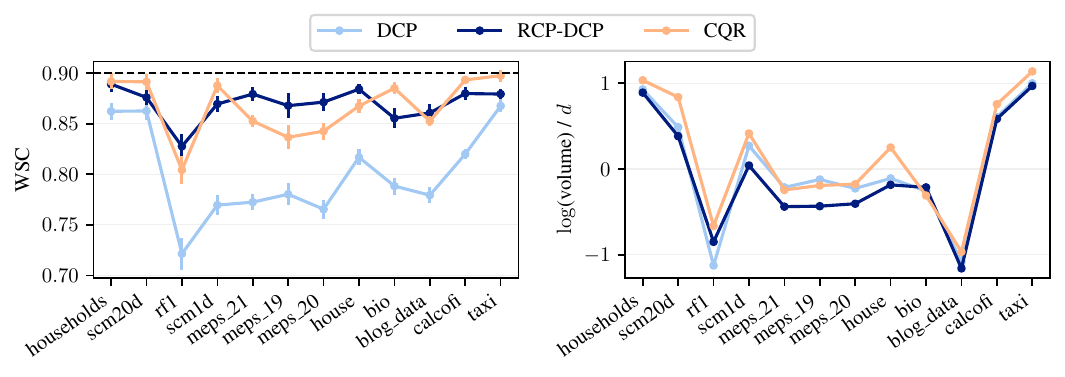}
    \caption{Worst slab coverage and (logarithm) median prediction set volume (scaled by $d$).}
    \label{fig:rcp_vs_base_cqr/cond_coverage_and_size}
  \end{figure}

\newpage

\subsection{Comparison of prediction set volumes}
\label{sec:volumes}
  \cref{table:mean_volume} shows the average volume obtained by the methods compared in \cref{sec:experiments}. Non-\RCP methods obtain a smaller average volume across all datasets. The larger average volume of \RCP is explained by the larger regions produced for instances with larger uncertainty.

  \begin{table}[H]
    \footnotesize
    \centering
    \caption{Mean prediction set volume per dataset.}
    \label{table:mean_volume}
    \begin{tabular}{lllllll}
    \toprule
    dataset & PCP & RCP-PCP & DCP & RCP-DCP & ResCP & RCP-ResCP \\
    \midrule
    households & 88.3 & 1.33e+02 & \bfseries 47.4 & 1.02e+02 & 1.81e+02 & 4.51e+02 \\
    scm20d & \bfseries 4.26e+05 & 7.92e+06 & 1.11e+06 & 3.95e+07 & 5.22e+05 & 7.15e+12 \\
    rf1 & 0.0274 & 0.190 & \bfseries 0.000562 & 4.35e+04 & 0.0276 & 7.94e+08 \\
    scm1d & \bfseries 1.92e+04 & 1.30e+08 & 2.30e+04 & 1.67e+08 & 6.27e+04 & 2.04e+15 \\
    meps\_21 & 1.65 & 10.0 & \bfseries 0.746 & 6.07 & 5.35 & 8.32e+12 \\
    meps\_19 & 90.0 & 3.27e+04 & \bfseries 3.64 & 3.27e+04 & 5.56 & 3.56e+22 \\
    meps\_20 & 1.68 & 5.50 & \bfseries 0.761 & 6.20 & 5.38 & 6.27e+13 \\
    house & 0.676 & 0.936 & \bfseries 0.519 & 0.751 & 2.92 & 3.88 \\
    bio & 0.579 & 1.12 & \bfseries 0.414 & 0.645 & 1.05 & 1.16 \\
    blog\_data & 0.459 & 8.45e+02 & \bfseries 0.143 & 6.37e+02 & 1.26 & 6.79e+21 \\
    calcofi & 3.47 & 4.12 & \bfseries 2.45 & 3.06 & 4.53 & 4.47 \\
    taxi & 9.21 & 9.63 & \bfseries 5.69 & 6.40 & 12.4 & 12.8 \\
    \bottomrule
    \end{tabular}
  \end{table}

  In contrast, \cref{table:median_volume} shows that \RCP obtains smaller regions across most datasets when comparing the median volume, avoiding outliers.

  \begin{table}[H]
    \footnotesize
    \centering
    \caption{Median prediction set volume per dataset.}
    \label{table:median_volume}
    \begin{tabular}{lllllll}
    \toprule
    dataset & PCP & RCP-PCP & DCP & RCP-DCP & ResCP & RCP-ResCP \\
    \midrule
    households & 67.4 & 56.5 & 39.2 & \bfseries 32.1 & 1.81e+02 & 1.67e+02 \\
    scm20d & 3.11e+04 & \bfseries 1.42e+04 & 7.03e+05 & 7.74e+04 & 5.22e+05 & 2.00e+05 \\
    rf1 & 0.0110 & 0.00525 & \bfseries 0.000583 & 33.1 & 0.0276 & 0.0231 \\
    scm1d & 1.16e+02 & \bfseries 2.05 & 1.01e+04 & 25.7 & 6.27e+04 & 1.70e+03 \\
    meps\_21 & 1.04 & 0.704 & 0.433 & \bfseries 0.227 & 5.35 & 2.42 \\
    meps\_19 & 3.85 & 0.754 & 2.10 & \bfseries 0.303 & 5.56 & 2.54 \\
    meps\_20 & 1.05 & 0.616 & 0.416 & \bfseries 0.254 & 5.38 & 2.51 \\
    house & 0.596 & 0.519 & 0.471 & \bfseries 0.386 & 2.92 & 2.67 \\
    bio & 0.507 & 0.435 & 0.374 & \bfseries 0.344 & 1.05 & 0.829 \\
    blog\_data & 0.229 & 0.209 & 0.0869 & \bfseries 0.0597 & 1.26 & 1.16 \\
    calcofi & 3.83 & 3.98 & 2.85 & \bfseries 2.77 & 4.53 & 4.75 \\
    taxi & 8.67 & 8.25 & \bfseries 5.22 & 5.27 & 12.4 & 10.1 \\
    \bottomrule
    \end{tabular}
  \end{table}

\newpage

\subsection{Case of ellipsoidal prediction sets}
\label{sec:simple_baseline}
  In this section, we will investigate how all parts of our proposed RCP method contribute to the performance. Additionally, we demonstrate the wider applicability of our approach: in this section, we use a different score and ellipsoid prediction sets. To achieve this, we modify our base models to predict parameters of the multivariate normal distribution. As a baseline method, we have selected CQR because of its popularity and ease of use.
 
  \cref{fig:pointplot/rcp_vs_qr/horizontal} demonstrates improvements of \RCP over simpler methods, with CQR serving as a strong contemporary alternative. Each of these simpler methods employs a consecutively more complex model and/or calibration procedure. The first part of the name corresponds to the base prediction model, and the second part after the dash denotes the calibration procedure:

  \begin{itemize}
      \item Const: the base model for these methods is a constant prediction of multivariate normal distribution parameters estimated on the train set.
      \item MLP: the base model is a multidimensional perceptron that predicts the parameters of multivariate normal distribution.
      \item CP: classic conformal prediction using full calibration set. We estimate a fixed prediction ellipsoid size using Mahalanobis distance-based nonconformity score.
      \item RCP: our usual RCP procedure where we split the calibration set and fit a quantile regression model to predict the $(1 - \alpha)$ quantile of the Mahalanobis score. Similarly to the other experiments, quantile regression is fit using MLP underlying model.
  \end{itemize}

  The alternative method CQR is based on quantile regression estimates for each dimension of the output variable. First, (univariate) conformalized quantile regression scores~\cite{romano2019conformalized} are computed for each dimension. Then, they are aggregated by taking the maximum score over each dimension, similarly to ResCP in the main text. The resulting prediction set in this case is a hyperrectangle. Its size is adaptive to the input, but the conformal correction is isotropic and constant for all input points. 

  %For the remaining methods, the model outputs parameters of a multivariate normal distribution. The score in this case is the Mahalanobis distance. The first part of the name before the dash states the type of the base model. ``Const'' is a constant prediction of the empirical mean and covariance matrix on the train set. ``MLP'' corresponds to a small neural network that predicts these parameters based on the input point $x$. The prediction sets for this group of methods are ellipsoids (not necessarily aligned with the axes).

  \begin{figure}[H]
    \centering
    \includegraphics[width=0.8\linewidth]{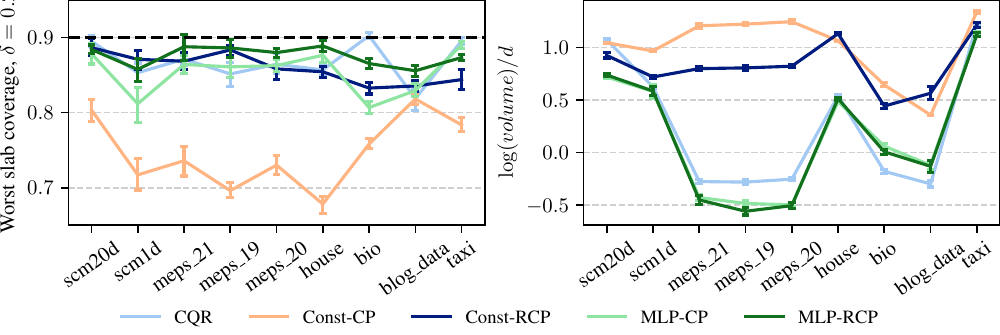}
    \caption{Worst slab coverage and logarithm of prediction set volume (divided by number of dimensions of the response).}
    \label{fig:pointplot/rcp_vs_qr/horizontal}
  \end{figure}

  The graphs on~\cref{fig:pointplot/rcp_vs_qr/horizontal} provide some important insights:
  \begin{itemize}
    \item Methods based on classic conformal prediction (Const-CP, MLP-CP) often struggle to maintain conditional coverage.
    \item \RCP improves conditional coverage: Const-RCP outperforms Const-CP in conditional coverage and set size.
    \item \RCP in combination with a better predictive model either maintains or improves conditional coverage and volume.
  \end{itemize}

\newpage

\subsection{Improved data efficiency using cross-validation}
\label{sec:CV_experiment}
  As explained in \cref{sec:experiments}, \RCP requires to divide the calibration dataset $\mathcal{D}$ into two parts, one to estimate $\widehat{\tau}$, and one for SCP.

  In this section, we consider a more data-efficient approach using the training dataset $\mathcal{D}_{\text{train}}$. Using $K$-fold cross-validation on $\mathcal{D}_{\text{train}}$, for each fold index $k$, we train a model on the $K - 1$ remaining folds and evaluate the conformity score on the fold $k$. This yields a dataset $\mathcal{D}_\tau$ of size $|\mathcal{D}_{\text{train}}|$ with inputs and their associated conformity scores based on which $\widehat{\tau}$ is estimated. This also removes the need to split the calibration dataset. An additional model is fitted on the complete training data set $\mathcal{D}_{\text{train}}$ to produce the non-rectified conformity scores.

  \cref{fig:rcp_qestimator_dataset/wsc} shows a comparison of learning $\widehat{\tau}$ on half the calibration dataset (cal), or using 10-fold cross-validation (CV). The cross-validation approach yields improved worst-slab coverage on most datasets.
  This improved conditional coverage comes at the computational cost of training $K$ additional models.

  \begin{figure}[H]
    \centering
    \includegraphics[width=0.95\linewidth]{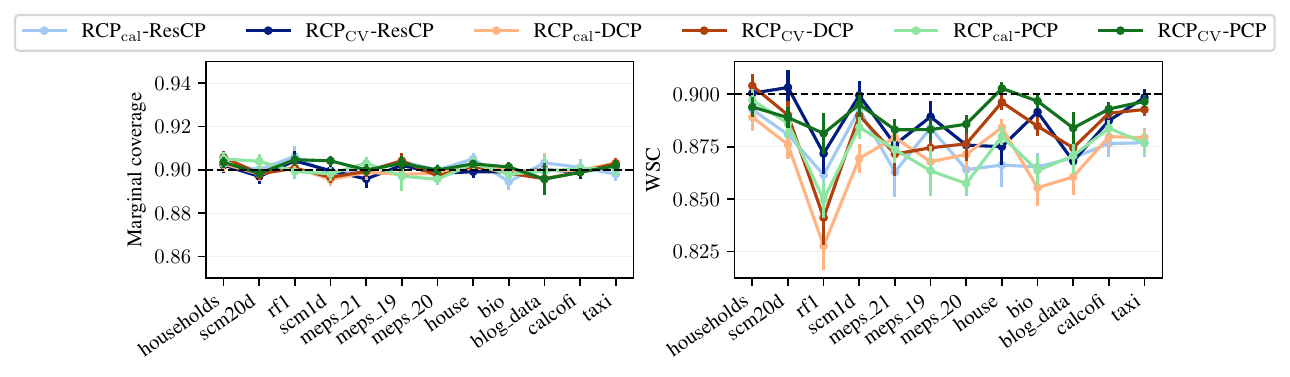}
    \caption{Worst-slab coverage of \RCP with $\widehat{\tau}$ trained on half the calibration dataset (cal) or using 10-fold cross-validation (CV).}
    \label{fig:rcp_qestimator_dataset/wsc}
  \end{figure}

\newpage

\subsection{Comparison with CPCG}
\label{sec:comparison_cpcg}
  We conduct an additional experiment comparing \RCP with Conditional Prediction with Conditional Guarantees (CPCG; \citet{gibbs2025conformal}). We evaluate \RCP using both the full calibration dataset (RCP$_\text{cal}$) and cross-validation (RCP$_\text{CV}$), as described in \cref{sec:CV_experiment}. All methods are run on CPU (AMD Ryzen Threadripper PRO 5965WX) with 6 CPU threads per experiment.
  
  \cref{table:wsc_cpcg} shows that all methods achieve comparable worst-slab coverage, close to the nominal level. However, \cref{table:time_cpcg} reveals a stark contrast in computational efficiency: CPCG is 200-100,000 times slower than RCP$_\text{cal}$ and 10-100 times slower than RCP$_\text{CV}$ overall. This significant overhead is because CPCG must solve an optimization problem involving the entire calibration set \textit{for each test instance}. Consequently, CPCG's computational demands become prohibitive for large calibration and test sets, hindering its practical application. Moreover, CPCG failed to find a solution on the \enquote{house} and \enquote{calcofi} datasets, precluding results for these cases. These factors highlight \RCP's substantial practical advantage in efficiency, especially for large-scale datasets.

  % \begin{table}[H]
  %   \centering
  %   \caption{Comparison of worst-slab coverage on multi-output datasets.}
  %   \label{table:wsc_cpcg}
  %   \begin{tabular}{lcccccc}
  %       \toprule
  %        & PCP & RCP-PCP & CPCG-PCP & DCP & RCP-DCP & CPCG-DCP \\
  %       \midrule
  %       scm20d     & 0.829 & 0.871 & 0.875 & 0.839 & 0.870 & 0.876 \\
  %       rf1        & 0.698 & 0.841 & 0.808 & 0.717 & 0.859 & 0.867 \\
  %       scm1d      & 0.735 & 0.864 & 0.895 & 0.765 & 0.868 & 0.889 \\
  %       meps\_21   & 0.762 & 0.880 & 0.909 & 0.802 & 0.900 & 0.894 \\
  %       meps\_19   & 0.754 & 0.879 & 0.892 & 0.768 & 0.855 & 0.890 \\
  %       meps\_20   & 0.729 & 0.849 & 0.898 & 0.772 & 0.849 & 0.866 \\
  %       bio        & 0.820 & 0.910 & 0.911 & 0.835 & 0.916 & 0.913 \\
  %       blog\_data & 0.782 & 0.874 & 0.905 & 0.796 & 0.867 & 0.885 \\
  %       \bottomrule
  %   \end{tabular}
  % \end{table}
  
  \begin{table}[H]
  \small
    \centering
    \caption{Comparison of worst-slab coverage on multi-output datasets.}
    \label{table:wsc_cpcg}
      \begin{tabular}{l|llll|llll}
        \toprule
        & PCP & RCP$_{\text{cal}}$-PCP & RCP$_{\text{CV}}$-PCP & CPCG-PCP & DCP & RCP$_{\text{cal}}$-DCP & RCP$_{\text{CV}}$-DCP & CPCG-DCP \\
        \midrule
        households & 0.825 & 0.905 & 0.899 & 0.888 & 0.853 & 0.891 & 0.900 & 0.900 \\
        scm20d & 0.830 & 0.892 & 0.891 & 0.897 & 0.877 & 0.877 & 0.868 & 0.899 \\
        rf1 & 0.731 & 0.830 & 0.877 & 0.838 & 0.715 & 0.863 & 0.827 & 0.872 \\
        scm1d & 0.758 & 0.882 & 0.895 & 0.910 & 0.756 & 0.902 & 0.896 & 0.882 \\
        meps\_21 & 0.739 & 0.874 & 0.904 & 0.881 & 0.789 & 0.881 & 0.879 & 0.905 \\
        meps\_19 & 0.762 & 0.875 & 0.867 & 0.880 & 0.788 & 0.884 & 0.889 & 0.878 \\
        meps\_20 & 0.731 & 0.842 & 0.871 & 0.890 & 0.719 & 0.880 & 0.884 & 0.892 \\
        house & 0.835 & 0.895 & 0.903 & / & 0.817 & 0.878 & 0.906 & / \\
        bio & 0.784 & 0.860 & 0.900 & 0.887 & 0.774 & 0.879 & 0.880 & 0.880 \\
        blog\_data & 0.770 & 0.877 & 0.893 & 0.886 & 0.749 & 0.844 & 0.888 & 0.888 \\
        calcofi & 0.810 & 0.889 & 0.888 & / & 0.828 & 0.885 & 0.892 & / \\
        taxi & 0.837 & 0.885 & 0.884 & 0.881 & 0.846 & 0.872 & 0.879 & 0.879 \\
        \bottomrule
      \end{tabular}
  \end{table}

  % \begin{table}[H]
  %   \centering
  %   \caption{Comparison of computational time (in seconds) on multi-output datasets.}
  %   \label{table:time_cpcg}
  %   \begin{tabular}{lcccccc}
  %       \toprule
  %        & PCP & RCP-PCP & CPCG-PCP & DCP & RCP-DCP & CPCG-DCP \\
  %       \midrule
  %       scm20d     & 0.847 & 15.9   & 5830  & 0.0170 & 19.2   & 5610  \\
  %       rf1        & 0.202 & 12.8   & 8860  & 0.0108 & 15.9   & 8780  \\
  %       scm1d      & 0.869 & 9.83   & 4730  & 0.0143 & 15.2   & 4730  \\
  %       meps\_21   & 0.103 & 0.315  & 4260  & 0.0129 & 0.842  & 4000  \\
  %       meps\_19   & 0.120 & 1.97   & 5940  & 0.0171 & 4.71   & 5610  \\
  %       meps\_20   & 0.121 & 0.259  & 5260  & 0.0138 & 1.43   & 5040  \\
  %       bio        & 0.268 & 0.295  & 45700 & 0.0338 & 0.464  & 43300 \\
  %       blog\_data & 0.276 & 0.294  & 24800 & 0.0312 & 0.991  & 28200 \\
  %       \bottomrule
  %   \end{tabular}
  % \end{table}

  \begin{table}[H]
  \small
    \centering
    \caption{Comparison of computational time (in seconds) on multi-output datasets.}
    \label{table:time_cpcg}
      \begin{tabular}{l|llll|llll}
        \toprule
        & PCP & RCP$_{\text{cal}}$-PCP & RCP$_{\text{CV}}$-PCP & CPCG-PCP & DCP & RCP$_{\text{cal}}$-DCP & RCP$_{\text{CV}}$-DCP & CPCG-DCP \\
        \midrule
        households & 0.258 & 0.604 & 104 & 8840 & 0.00759 & 0.531 & 104 & 8164 \\
        scm20d & 2.63 & 4.22 & 772 & 6409 & 0.0182 & 0.852 & 766 & 6012 \\
        rf1 & 0.667 & 1.35 & 340 & 10682 & 0.00836 & 0.348 & 339 & 9674 \\
        scm1d & 2.27 & 3.57 & 1209 & 4971 & 0.0133 & 0.867 & 1205 & 4692 \\
        meps\_21 & 0.236 & 0.607 & 581 & 6283 & 0.0123 & 0.261 & 581 & 6031 \\
        meps\_19 & 0.272 & 0.515 & 493 & 6411 & 0.0123 & 0.184 & 492 & 6128 \\
        meps\_20 & 0.255 & 0.520 & 621 & 7147 & 0.0119 & 0.238 & 621 & 7032 \\
        house & 0.315 & 0.594 & 1034 & / & 0.0159 & 0.327 & 1033 & / \\
        bio & 0.630 & 1.22 & 3161 & 79422 & 0.0279 & 0.782 & 3163 & 63178 \\
        blog\_data & 0.752 & 0.850 & 1119 & 41192 & 0.0336 & 0.155 & 1121 & 43289 \\
        calcofi & 0.699 & 1.02 & 456 & / & 0.0356 & 0.200 & 455 & / \\
        taxi & 0.680 & 1.17 & 866 & 77828 & 0.0269 & 0.276 & 866 & 70139 \\
        \bottomrule
      \end{tabular}
  \end{table}

\newpage

\section{Proofs}
\label{suppl:examples}
  % !TEX root = ../main.tex

\subsection{Proof for the multiplicative example}
\label{suppl:examplesA}
Let $p=1-\alpha$ and $q(x)=Q_p(\textup{P}_{\mathbf{V}\mid X=x})$.
For a positive affine transformation $S(x,y)=a(x)V(x,y)+b(x)$ with
$a(x)>0$, its conditional CDF is
\[
  \prob(S(X,Y)\le s\mid X=x)
  =\prob\left(\mathbf{V}\le \frac{s-b(x)}{a(x)}\,\middle|\,X=x\right).
\]
The right-hand side is at least $p$ at $s=a(x)q(x)+b(x)$ and is strictly
less than $p$ for every smaller $s$, by the definition of the lower
quantile. If $a(x)q(x)+b(x)=c$ for $\textup{P}_X$-almost every $x$, averaging
these inequalities over $X$ gives the same lower quantile $c$ marginally.

For \Cref{example:linear}, $q(x)>0$, $a(x)=1/q(x)$, and $b(x)=0$.
Thus the conditional and marginal lower $p$-quantiles of
$\tilde V(X,Y)$ are both $1$, proving \eqref{eq:mc_cc}.

\subsection{Proof for the additive example}
\label{suppl:examplesB}
For \Cref{example:additive}, use the same affine transformation with
$a(x)=1$ and $b(x)=-q(x)$. The argument in \Cref{suppl:examplesA} gives
both lower $p$-quantiles equal to $0$.

In both examples, thresholding the rectified score at its common
quantile is equivalent to thresholding $\mathbf V$ at $q(x)$. The resulting
conditional coverage lies between $p$ and
$p+\prob(\mathbf V=q(x)\mid X=x)$ since
\[
  \prob(\mathbf V<q(x)\mid X=x)\le p
  \le \prob(\mathbf V\le q(x)\mid X=x).
\]
When there is no conditional atom at $q(x)$ for $\textup{P}_X$-almost
every $x$, both conditional and marginal coverage equal $p$.

\subsection{Proof of \texorpdfstring{\Cref{thm:coverage:marginal}}{}}
\label{suppl:marginal}
  % !TEX root = ../main.tex

\begin{proof}[Proof of \Cref{thm:coverage:marginal}]
  Condition on $\mathcal{D}_{\mathrm{fit}}$, the data used to construct $V$
  and $\widehat{\tau}$. Attach to each rectified score an independent uniform
  tie-breaker, also independent of all the data. Let $R$ be the rank of the
  test pair of score and tie-breaker when the $n+1$ pairs are ordered
  lexicographically with the smallest first. Conditional exchangeability
  and the absence of ties between pairs make $R$ uniform on
  $\{1,\ldots,n+1\}$.

  Write $k_\alpha=\lceil(n+1)(1-\alpha)\rceil$. When $k_\alpha\le n$, the
  threshold in~\eqref{eq:rcp_empirical} is the $k_\alpha$th calibration
  score in increasing order. If $R\le k_\alpha$, at most $k_\alpha-1$
  calibration scores are strictly smaller than the test score, so the
  prediction set covers the test point. The implication is an equivalence
  when the scores are distinct. When $k_\alpha=n+1$, the threshold is
  $+\infty$, so coverage and $R\le k_\alpha$ both occur surely. Hence
  \[
    \prob\bigl(Y_{n+1}\in\mathcal C_\alpha(X_{n+1})
              \mid\mathcal{D}_{\mathrm{fit}}\bigr)
    \ge \frac{k_\alpha}{n+1}\ge 1-\alpha,
  \]
  with equality in the first inequality for almost surely distinct
  scores. In this case $k_\alpha/(n+1)<1-\alpha+1/(n+1)$ by the ceiling
  definition. The same bounds hold after averaging over the fitting data.
\end{proof}

\subsection{Proof of \texorpdfstring{equality~\eqref{eq:key-relation}}{}}
\label{suppl:cond-quantile}
  % !TEX root = ../main.tex

\begin{theorem}\label{thm:conditional:varphi}
  Assume \Cref{ass:tau}, set $p=1-\alpha$, and suppose that
  $\tau_\star(x)=Q_p(\textup{P}_{\mathbf{V}_\varphi\mid X=x})$
  belongs to $\mathbb T$ for $\textup{P}_X$-almost every $x$.
  With $\tilde V_\star(x,y)=f_{\tau_\star(x)}^{-1}(V(x,y))$, for
  $\textup{P}_X$-almost every $x$,
  \[
    Q_p(\textup{P}_{\mathbf V\mid X=x})=f_{\tau_\star(x)}(\varphi),
    \qquad
    Q_p(\textup{P}_{\tilde V_\star(X,Y)\mid X=x})
    =\varphi=Q_p(\textup{P}_{\tilde V_\star(X,Y)}).
  \]
  The oracle set \eqref{eq:rcp2} has conditional coverage at least $p$
  for almost every $x$, and marginal coverage at least $p$.
  Both equal $p$ if $\mathbf V_\varphi$ has no conditional atom at
  $\tau_\star(x)$ for $\textup{P}_X$-almost every $x$.
\end{theorem}

\begin{proof}
  Fix such an $x$, set $t=\tau_\star(x)$ and $q=\tilde f_\varphi(t)$,
  and write $H_x(u)=\prob(\mathbf V_\varphi\le u\mid X=x)$.
  Since $\mathbf V=\tilde f_\varphi(\mathbf V_\varphi)$ and
  $\tilde f_\varphi$ is strictly increasing,
  \[
    \prob(\mathbf V\le q\mid X=x)=H_x(t)\ge p.
  \]
  For any $z<q$, continuity relative to $\mathbb T$ gives $\eta>0$
  such that $\tilde f_\varphi(w)>z$ for all $w\in\mathbb T$ with
  $|w-t|<\eta$. Monotonicity gives the same inequality for every
  $w\in\mathbb T$ with $w\ge t$. Thus $\mathbf V\le z$ implies
  $\mathbf V_\varphi\le t-\eta$, and
  $\prob(\mathbf V\le z\mid X=x)\le H_x(t-\eta)<p$.
  This proves $Q_p(\textup{P}_{\mathbf V\mid X=x})=q=f_t(\varphi)$
  and the inverse representation in \eqref{eq:definition-tau}.

  Strict monotonicity of $f_t$ gives
  \[
    \prob(\tilde V_\star(X,Y)\le s\mid X=x)
    =\prob(\mathbf V\le f_t(s)\mid X=x).
  \]
  At $s=\varphi$ this probability is $H_x(t)\ge p$, and for
  $s<\varphi$ it is strictly less than $p$ because $f_t(s)<q$.
  Integrating over $X$ gives the same inequalities for the marginal
  CDF, so $\varphi$ is also its lower $p$-quantile.
  At the oracle threshold the conditional coverage is $H_x(t)$.
  The lower quantile property gives $H_x(t^-)\le p\le H_x(t)$.
  If the conditional law has no atom at $t$, both CDF values equal $p$,
  and averaging yields marginal coverage $p$ as well.
\end{proof}

\subsection{Proof of \texorpdfstring{\Cref{thm:coverage:conditional}}{}}
\label{suppl:conditional:validity}
  % !TEX root = ../main.tex

We condition throughout on $\mathcal{D}_{\mathrm{fit}}$, the data used to fit the
score $V$ and the quantile estimate $\widehat{\tau}$. The final calibration
observations and the test pair are i.i.d. and independent of these data.
In what follows, $F$, $F_x$, $\beta$, $k_\alpha$, and $\mu$ are as in
\Cref{thm:coverage:conditional}.

\begin{lemma}\label{lem:bound:expec-1alphaU}
  Let $U_{(k)}$ be the $k$th order statistic of $n$ independent uniform random
  variables on $[0,1]$. For $\beta\in[0,1]$ and $1\le k\le n$, set
  $M_{n,k}(\beta)=\E[(\beta-U_{(k)})_+]$, where $z_+=\max(z,0)$. Then
  \begin{equation}\label{eq:eq:expect-beta-Uk}
    M_{n,k}(\beta)
    = \beta I_\beta(k,n+1-k)
      - \frac{k}{n+1} I_\beta(k+1,n+1-k),
  \end{equation}
  where $I_\beta(a,b)=B_\beta(a,b)/B(a,b)$ is the regularized incomplete beta
  function, $B_\beta(a,b)=\int_0^\beta u^{a-1}(1-u)^{b-1}\rmd u$, and
  $B(a,b)=B_1(a,b)$.
  The mean of $U_{(k)}$ is $\mu_k=k/(n+1)$, and its variance is
  $\mu_k(1-\mu_k)/(n+2)$.
\end{lemma}

\begin{proof}
  The event $U_{(k)}\le u$ means at least $k$ of the $n$ uniforms
  are at most $u$, so its probability is
  $\sum_{j=k}^n\binom{n}{j}u^j(1-u)^{n-j}$.
  Differentiating this sum cancels adjacent terms, leaving
  $\frac{n!}{(k-1)!(n-k)!}u^{k-1}(1-u)^{n-k}$.
  With $a=k$ and $b=n+1-k$, normalization makes the density
  $u^{a-1}(1-u)^{b-1}/B(a,b)$ on $(0,1)$. Integrating $(\beta-u)_+$
  against it gives
  \[
    M_{n,k}(\beta)
    = \frac{\beta B_\beta(a,b)-B_\beta(a+1,b)}{B(a,b)}
    = \beta I_\beta(a,b)-\frac{a}{a+b}I_\beta(a+1,b).
  \]
  The last equality uses $B(a+1,b)/B(a,b)=a/(a+b)$. The same beta ratio
  gives $\E U_{(k)}=a/(a+b)$ and
  \[
    \E[U_{(k)}^2]=\frac{B(a+2,b)}{B(a,b)}
      =\frac{a(a+1)}{(a+b)(a+b+1)}.
  \]
  Subtracting the square of the mean gives the variance.
\end{proof}

In particular, integration by parts on $[0,\beta]$ for integers $i\ge 0$
and $j\ge 1$ gives
\[
  \frac{B_\beta(i+1,j+1)}{i!j!}
  = \frac{\beta^{i+1}(1-\beta)^j}{(i+1)!j!}
    + \frac{B_\beta(i+2,j)}{(i+1)!(j-1)!}.
\]
The boundary term contributes when $0<\beta<1$. For example,
$M_{2,1}(1/2)=\int_0^{1/2}(1/2-u)2(1-u)\rmd u=5/24$.

\paragraph{Two-sided finite-sample bound.}
Under the hypotheses of \Cref{thm:coverage:conditional}, for
$\alpha\ge 1/(n+1)$ and $\mathrm{P}_X$-almost every $x$, the conditional
coverage $c_\alpha(x)$ satisfies
\begin{equation}\label{eq:conditional-exact-moments}
  -\mathrm{L} M_{n,k_\alpha}(\beta)
  \le c_\alpha(x)-(1-\alpha)-\epsilon_{\widehat{\tau}}(x)
  \le \mathrm{L} M_{n,n+1-k_\alpha}(1-\beta).
\end{equation}

\begin{proof}[Proof of \Cref{thm:coverage:conditional} and \eqref{eq:conditional-exact-moments}]
  Let $Q$ be the calibration threshold in \eqref{eq:rcp_empirical}.
  If $0<\alpha<1/(n+1)$, $k_\alpha=n+1$ and $Q=+\infty$, so the
  set covers every finite score. Otherwise $Q$ is the $k_\alpha$th
  smallest calibration score.
  The continuity of $F$ implies that $F(Q)$ has the law of $U_{(k_\alpha)}$.
  Independence of calibration and test data gives
  $c_\alpha(x)=\E[F_x(Q)]$ conditional on $\mathcal{D}_{\mathrm{fit}}$.

  Set $g_x=F_x\circ F^{-1}$ on $(0,1)$, where $F^{-1}$ is the generalized
  inverse, and extend $g_x$ continuously to $[0,1]$ using its Lipschitz
  property. To handle flat parts of $F$, let $S=\tilde V(X,Y)$ and $U=F(S)$.
  By continuity of $F$, $U$ is uniform on $(0,1)$. Since $F^{-1}(U)\le S$
  and $F^{-1}(U)$ has the same distribution as $S$, we have $S=F^{-1}(U)$
  almost surely. This equality holds conditionally on $X=x$ for
  $\mathrm{P}_X$-almost every $x$. For these $x$, the quantile equivalence
  gives
  \[
    F_x(t)=\prob(U\le F(t)\mid X=x),\qquad t\in\R.
  \]
  Since $F(F^{-1}(u))=u$ for $0<u<1$, $g_x(u)$ is the conditional CDF of
  $U$ there. Its limits at $0$ and $1$ are $0$ and $1$, respectively,
  because $U\in(0,1)$ almost surely. Thus $F_x(t)=g_x(F(t))$ for every
  $t$, including at $F(t)\in\{0,1\}$. Also,
  $F_x(\varphi)=1-\alpha+\epsilon_{\widehat{\tau}}(x)$ by
  \eqref{eq:def:V-varphi} and the monotonicity in \Cref{ass:tau}.

  Write $U_{(k_\alpha)}$ as $U_k$ temporarily. Because $g_x$ is nondecreasing
  and $\mathrm{L}$-Lipschitz,
  \[
    -\mathrm{L}(\beta-U_k)_+
    \le g_x(U_k)-g_x(\beta)
    \le \mathrm{L}(U_k-\beta)_+.
  \]
  Taking expectations and using $1-U_k\stackrel{d}{=}U_{(n+1-k_\alpha)}$
  proves \eqref{eq:conditional-exact-moments}.

  By \Cref{lem:bound:expec-1alphaU}, the mean and variance of $U_k$ are
  $\mu$ and $\sigma^2=\mu(1-\mu)/(n+2)$. Its centered positive and negative parts
  both have expectation $\E|U_k-\mu|/2\le \sigma/2$. By subadditivity of
  the positive part,
  \begin{align*}
    M_{n,k_\alpha}(\beta)&\le(\beta-\mu)_++\sigma/2,\\
    M_{n,n+1-k_\alpha}(1-\beta)&\le(\mu-\beta)_++\sigma/2 .
  \end{align*}
  Thus the absolute deviation from $F_x(\varphi)$ is at most
  $\mathrm{L}(|\beta-\mu|+\sigma/2)$. Since $\mu(1-\mu)\le 1/4$, this
  gives \eqref{eq:def:conditional-bound}.
\end{proof}

\paragraph{Order of the calibration term at the oracle.}
Consider two equally likely values of $X$ with rectified scores on $[0,1]$.
Their conditional CDFs on this interval are $g_0(u)=u-h(u)/2$ and
$g_1(u)=u+h(u)/2$, where $h(u)=\min\{u,|u-1/2|,1-u\}$. Both CDFs
are strictly increasing and $3/2$-Lipschitz on $[0,1]$. Their mixture is
uniform and both have median $\varphi=1/2$. These scores are obtained with
the additive transformation $f_t(v)=v+t$ and the oracle $\tau_\star=0$.
For $\alpha=1/2$ and $n=3$, $U=U_{(k_\alpha)}$ is $\mathrm{Beta}(2,2)$.
Its density and $h$ are both symmetric about $1/2$, so
\[
  \E[h(U)]=2\left(\int_0^{1/4}6u^2(1-u)\rmd u
                   +\int_{1/4}^{1/2}6u(1-u)(1/2-u)\rmd u\right)
  =\tfrac{17}{128}.
\]
Consequently, $c_\alpha(0)=1/2-\E[h(U)]/2=111/256$. The proposed
lower bound $1-\alpha-\alpha\mathrm{L}\beta^{n+1}$ would equal
$29/64=116/256$, exceeding that coverage.

More generally, for odd $n=2a-1$, write $U=U_{(k_\alpha)}\sim\mathrm{Beta}(a,a)$.
By symmetry and $\frac{\rmd}{\rmd u}[u(1-u)]^a
=a(1-2u)u^{a-1}(1-u)^{a-1}$,
\[
  \E|U-1/2|
  =\frac{2}{B(a,a)}\int_0^{1/2}(1/2-u)u^{a-1}(1-u)^{a-1}\rmd u
  =\frac{\binom{2a}{a}}{2\cdot 4^a}.
\]
This is of order $n^{-1/2}$ by Stirling's formula. The difference
$|u-1/2|-h(u)$ vanishes on $[1/4,3/4]$. Outside this interval it is at
most $|u-1/2|\le 4(u-1/2)^2$. The expected difference is therefore bounded
by $4\var(U)=1/(n+2)$. Since $1/2-c_\alpha(0)=\E[h(U)]/2$, the conditional
undercoverage is of order $n^{-1/2}$ even at the oracle.

\subsection{Pointwise control of \texorpdfstring{$\epsilon_{\tau}$}{}}
\label{suppl:epsilon}
  % !TEX root = ../main.tex

\begin{theorem}\label{eq:quantile-conditional-pinball-loss}
  Fix $x\in\XC$ and condition on $\mathcal{D}_{\mathrm{fit}}$ as in
  \Cref{sec:theory}. Assume that $V_\varphi(X,Y)$ has a finite first
  absolute moment conditional on $X=x$ and that its conditional CDF
  $F_{V_\varphi\mid X=x}$ is $D_x$-Lipschitz for some $D_x>0$. Then
  \begin{equation*}
    \left|\epsilon_\tau(x)\right|
    \le \sqrt{2D_x\left(
      \mathcal{L}_x(\tau(x))-\mathcal{L}_x(\tau_\star(x))
    \right)}.
  \end{equation*}
\end{theorem}

\begin{proof}
  With the fitted data held fixed, write
  $\mathcal{L}_x(t)=\E[\rho_{1-\alpha}(V_\varphi(X,Y)-t)\mid X=x]$.
  The CDF is continuous and the pinball loss has bounded slopes, so
  differentiating under the expectation yields
  \begin{equation*}
    \mathcal{L}_x'(t)=F_{V_\varphi\mid X=x}(t)-(1-\alpha).
  \end{equation*}
  The derivative vanishes at $\tau_\star(x)$ by continuity of the CDF
  and the definition of the quantile, so $\tau_\star(x)$ minimizes the
  convex loss $\mathcal{L}_x$.
  The derivative is $D_x$-Lipschitz. Integrating its Lipschitz bound gives,
  for all $s,t\in\R$,
  \begin{equation*}
    \mathcal{L}_x(s)
    \le \mathcal{L}_x(t)+\mathcal{L}_x'(t)(s-t)+\frac{D_x}{2}(s-t)^2.
  \end{equation*}
  Substituting $s=t-\mathcal{L}_x'(t)/D_x$ and using that $\tau_\star(x)$
  minimizes $\mathcal{L}_x$ gives
  \begin{equation*}
    \mathcal{L}_x(\tau_\star(x))
    \le \mathcal{L}_x(t)-\frac{|\mathcal{L}_x'(t)|^2}{2D_x}.
  \end{equation*}
  The result follows by taking $t=\tau(x)$, where
  $\mathcal{L}_x'(t)=\epsilon_\tau(x)$.
\end{proof}

\subsection{Proof of the local quantile regression bound}
\label{suppl:cdf}
  % !TEX root = ../main.tex

\begin{proof}[Proof of \Cref{cor:epsilon-tau:local-cdf}]
  Hold the fitted score fixed and fix $x$. Write $q=1-\alpha$,
  $H_x(v)=F_{V_\varphi\mid X=x}(v)$, and
  $K_{\max}=K_{h_X}(0)$. Let $W=K_{h_X}(\|x-X\|)$ and
  $W_i=K_{h_X}(\|x-X_i'\|)$ for $i\in[\ccount]$. The fitting observations
  are independent draws from the law of $(X,Y)$ with the score held fixed.
  By \Cref{ass:kernel-cdf},
  the population weighted distribution satisfies
  \begin{equation*}
    \left|
      \frac{\E[W\1_{V_\varphi(X,Y)\le v}]}{C_{h_X}(x)}-H_x(v)
    \right|
    \le \frac{\E[W\left|H_X(v)-H_x(v)\right|]}{C_{h_X}(x)}
    \le b_{h_X}(x).
  \end{equation*}

  For each threshold $v$, define
  $Z_i(v)=W_i(\1_{V_\varphi(X_i',Y_i')\le v}-q)$.
  These variables are independent, take values in
  $[-qK_{\max},(1-q)K_{\max}]$, and have variance at most $\E[W^2]$.
  If the sum of the weights is positive, every weighted pinball minimizer
  $\widehat{\tau}(x)$ satisfies
  \begin{equation*}
    \sum_{i=1}^{\ccount} W_i\1_{V_\varphi(X_i',Y_i')<\widehat{\tau}(x)}
    \le q\sum_{i=1}^{\ccount} W_i
    \le \sum_{i=1}^{\ccount} W_i\1_{V_\varphi(X_i',Y_i')\le\widehat{\tau}(x)}.
  \end{equation*}
  This follows by comparing the left and right derivatives of the weighted
  pinball loss at its minimizer. The displayed inequalities also hold
  if all weights vanish.

  Take $\delta>b_{h_X}(x)$. If $q-\delta>0$, continuity gives a threshold
  $v_-$ such that $H_x(v_-)=q-\delta$. On the event
  $\epsilon_{\widehat{\tau}}(x)<-\delta$, monotonicity implies
  $\widehat{\tau}(x)<v_-$, so $\sum_i Z_i(v_-)\ge 0$. The population bound gives
  $\E Z_i(v_-)\le -C_{h_X}(x)(\delta-b_{h_X}(x))$.
  If $q+\delta<1$, choose $v_+$ such that $H_x(v_+)=q+\delta$.
  On the event $\epsilon_{\widehat{\tau}}(x)>\delta$ we have
  $\widehat{\tau}(x)>v_+$, so $\sum_i Z_i(v_+)\le 0$,
  while $\E Z_i(v_+)\ge C_{h_X}(x)(\delta-b_{h_X}(x))$.
  A lower tail with $q-\delta\le 0$ or an upper tail with $q+\delta\ge 1$
  is empty, because $0\le H_x(\widehat{\tau}(x))\le 1$.

  Put $s=C_{h_X}(x)(\delta-b_{h_X}(x))>0$. Centering each $Z_i(v)$
  bounds its absolute value by $K_{\max}$. Bernstein's inequality bounds
  each of the two tail probabilities by
  \begin{equation*}
    \exp\left(-\frac{\ccount s^2}{2(\E[W^2]+K_{\max}s/3)}\right).
  \end{equation*}
  With $\ell=\log(2/\gamma)$, take
  $s=\sqrt{2\E[W^2]\ell/\ccount}+2K_{\max}\ell/(3\ccount)$.
  This gives $\ccount s^2\ge 2\ell(\E[W^2]+K_{\max}s/3)$, so the bound
  for each tail is at most $\gamma/2$.
  Summing the two probabilities proves the result.
\end{proof}

\subsection{Asymptotic conditional validity}
\label{suppl:asympt}
  % !TEX root = ../main.tex

Fix $p=1-\alpha\in(0,1)$. Assume \Cref{ass:tau} and \Cref{ass:tau-in-T}.
The final calibration observations $\mathcal{D}_{\mathrm{cal}}$ and the test
pair $(X,Y)$ are independent draws from the same population, independent
of the fitting data $\mathcal{D}_{\mathrm{fit}}$ used to construct $V$ and
$\widehat{\tau}$. Use $F$ and $F_x$ as in \Cref{thm:coverage:conditional}
and set $\beta=F(\varphi)$. For $n$ calibration observations, write
$k_\alpha=\lceil(n+1)p\rceil$ and let $Q$ be the $k_\alpha$th smallest
calibration score. We take $n$ large enough that $k_\alpha\le n$.
The coverage given both data splits is
\[
  C(x)=\prob\bigl(Y\in\mathcal C_\alpha(x)
       \mid X=x,\mathcal{D}_{\mathrm{fit}},\mathcal{D}_{\mathrm{cal}}\bigr)
      =F_x(Q).
\]
The coverage in \Cref{thm:coverage:conditional} is
$c_\alpha(x)=\E[C(x)\mid\mathcal{D}_{\mathrm{fit}}]$.
The error $\epsilon_{\widehat{\tau}}(x)=F_x(\varphi)-p$ is defined
in~\eqref{eq:def:epsilon-tau}.

\begin{proposition}
  Suppose $F$ is almost surely continuous.
  With $\E_X$ denoting expectation over the independent test covariate, we have
  \begin{equation}\label{eq:integrated-conditional-coverage}
    \E\bigl[\E_X|C(X)-p|\mid\mathcal{D}_{\mathrm{fit}}\bigr]
    \le 2\E_X|\epsilon_{\widehat{\tau}}(X)|
         +\frac{1}{n+1}+\frac{1}{2\sqrt{n+2}}.
  \end{equation}
  In particular, if $\E|\epsilon_{\widehat{\tau}}(X)|\to 0$ and $n\to\infty$,
  then $C(X)\to p$ in probability.
\end{proposition}

\begin{proof}
  Given $\mathcal{D}_{\mathrm{fit}}$ and $\mathcal{D}_{\mathrm{cal}}$, the difference
  $F_x(Q)-F_x(\varphi)$ has the same sign for every $x$ by monotonicity. Thus
  \[
    \E_X|F_X(Q)-F_X(\varphi)|=|F(Q)-\beta|.
  \]
  The triangle inequality yields
  $\E_X|C(X)-p|\le \E_X|\epsilon_{\widehat{\tau}}(X)|+|F(Q)-\beta|$.
  Conditional on $\mathcal{D}_{\mathrm{fit}}$, the probability integral
  transform makes $F(Q)$ the $k_\alpha$th order statistic of $n$ independent
  uniforms. By \Cref{lem:bound:expec-1alphaU}, its mean is
  $\mu=k_\alpha/(n+1)$ and its variance is
  $\mu(1-\mu)/(n+2)\le 1/(4(n+2))$. Since $0\le\mu-p<1/(n+1)$ and
  $\beta-p=\E_X\epsilon_{\widehat{\tau}}(X)$, it follows that
  \begin{align*}
    \E\bigl[|F(Q)-\beta|\mid\mathcal{D}_{\mathrm{fit}}\bigr]
    &\le |\beta-\mu|+\frac{1}{2\sqrt{n+2}}\\
    &\le \E_X|\epsilon_{\widehat{\tau}}(X)|
      +\frac{1}{n+1}+\frac{1}{2\sqrt{n+2}}.
  \end{align*}
  This proves~\eqref{eq:integrated-conditional-coverage}. Taking expectations
  and applying Markov's inequality gives the asserted convergence. Jensen's
  inequality gives the same convergence for coverage averaged over either
  or both data splits.
\end{proof}

\begin{theorem}
  Under the setup above, assume \Cref{ass:kernel-cdf}. Use the local
  quantile estimate in~\eqref{eq:def:reg-tau} with $m$ i.i.d. observations
  from the same population. The score $V$ is held fixed or fitted
  independently of these observations and of calibration and test data.
  Conditional on $V$, assume that the CDF $H_x$ of $V_\varphi(X,Y)$
  given $X=x$ is almost surely continuous for $\textup{P}_X$-almost
  every $x$.
  Let $d_X$ be the dimension of $X$, assume its density $f_X$ is bounded and
  continuous, and take a bounded nonnegative kernel profile $K_1$ with
  $0<\int_0^\infty r^{d_X-1} K_1(r)\,\rmd r<\infty$.
  For each $\rho>0$,
  \[
    \lim_{h_X\downarrow 0}\limsup_{n,m\to\infty}
      \prob\bigl(|C(X)-p|>\rho\bigr)=0.
  \]
  The same convergence holds jointly as $n,m\to\infty$ if
  $h_X\to0$ and $m h_X^{d_X}\to\infty$.
\end{theorem}

\begin{proof}
  First, the continuity of $H_x$ ensures that the rectified score has no
  atoms conditional on $\mathcal{D}_{\mathrm{fit}}$ and almost every $x$.
  Indeed, at fixed $x$ and fitted $\widehat{\tau}$, the rectified score is
  obtained from $V_\varphi(x,Y)$ by the injective map
  $f_{\widehat{\tau}(x)}^{-1}\circ\tilde f_\varphi$.
  Integrating over $X$ gives continuity of $F$.

  Condition on the fitted score $V$ and take $x$ with $f_X(x)>0$ where
  $H_x$ is continuous. Use $h=h_X$ and the notation in
  \Cref{cor:epsilon-tau:local-cdf}. With $K_h(r)=h^{-1}K_1(r/h)$, a
  change of variables and dominated convergence give
  \[
    h^{1-d_X}C_h(x)
    =\int_{\R^{d_X}} K_1(\|u\|) f_X(x+h u)\,\rmd u
    \longrightarrow f_X(x)\int_{\R^{d_X}}K_1(\|u\|)\,\rmd u > 0 .
  \]
  In the rescaled numerator of $b_h(x)$ the integrand has the additional
  factor $\min\{1,\mathrm{M}h\|u\|\}$, so dominated convergence yields
  $b_h(x)\to 0$.

  Since $\E[W_h^2]\le K_h(0)C_h(x)$, for fixed $\gamma$ the stochastic
  terms in \Cref{cor:epsilon-tau:local-cdf} tend to zero whenever
  $K_h(0)/(m C_h(x))\to0$. This holds for each sufficiently small fixed
  $h$ as $m\to\infty$, and in the joint regime $h\to0$ with
  $m h^{d_X}\to\infty$. For each $\delta>0$ take $h$ small enough that
  $b_h(x)<\delta$. The proposition then gives an eventual upper bound
  $\gamma$ for $\prob(|\epsilon_{\widehat{\tau}}(x)|>\delta)$ in both
  regimes. Since $\gamma$ is arbitrary, the error converges to zero in
  probability at $\textup{P}_X$-almost every $x$ in the stated limits.

  Since the errors are bounded by one, we obtain
  $\E|\epsilon_{\widehat{\tau}}(X)|\to0$ in the joint limit by dominated
  convergence. In the iterated limit, the reverse Fatou inequality first
  bounds the inner limit superior by the integral of the pointwise
  limit superiors, and dominated convergence then applies as $h\downarrow0$.
  These arguments also integrate over the independently fitted score,
  held fixed when taking the limits. The result follows
  from~\eqref{eq:integrated-conditional-coverage} and Markov's inequality.
\end{proof}

\newpage

\section{Details on Experimental Setup}
\label{suppl:details}
  % !TEX root = ../main.tex

  This section aims to provide additional details on our experimental setup and implementation of the~\RCP\ algorithm.

\paragraphformat{Models.}
  To facilitate fair comparison of different uncertainty estimation methods, we assume that the base prediction models are already trained. We focus on the regression problem and aim to construct prediction sets for these pre-trained models. All our models are based on a fully connected neural network of three hidden layers with 100 neurons in each layer and ReLU activations. We consider three types of base models with appropriate output layers and loss functions: the mean squared error for the \textit{mean predictor}, the pinball loss for the \textit{quantile predictor} or the negative log-likelihood loss for the \textit{mixture predictor}. Training is performed with Adam optimizer.

  Each dataset is split randomly into train, calibration, and test parts. We reserve 2048 points for calibration and the remaining data is split between 70\% for training and 30\% for testing. Each dataset is shuffled and split 10 times to replicate the experiment. This way we have 10 different models for each dataset and these models' prediction are used by every method that is tailored to the corresponding model type to estimate uncertainty. One fifth of the train dataset is reserved for early stopping.

\paragraphformat{RCP$_{\mathrm{MLP}}$.}
  This variation reserves a part ($50\%$) of the original calibration set to train a quantile regression model for the $(1-\alpha)$-level quantile of the scores $V$. We again use a three hidden layers with 100 units per layers for that task. The remaining half of the calibration set forms the ``proper calibration set'' and is used to compute the conformal correction.

% \subsection{RCP-C}
% For this method we use $25\%$ of the calibration data for the Nadaraya-Watson based estimate of score CDF and $25\%$ to compute the conformal correction. The third part of $25\%$ of calibration points are used to tune the bandwidth. We use RBF kernel functions, isotropic for $X$, so we need to tune two real-valued bandwidth parameters. We employ Dual Annealing method with bandwidths limited to $[10^{-4}, 10]$. After the bandwidth is tuned, use the calibration data as before: split in half to estimate the conditional quantile and conformal correction.

\paragraphformat{RCP$_{\mathrm{local}}$.} 
  The local quantile regression variant is similar to the previous one, so we use the same splitting of the available calibration data. Since only one bandwidth needs to be tuned, we use a simple grid search on a log-scale grid in the interval $[10^{-3}, 1]$.

\paragraphformat{Datasets.}
%\label{subsec:datasets}
  \cref{table:datasets} presents characteristics of datasets from~\citep{Tsoumakas2011-wf,Feldman2023-cc,wang2023probabilistic}, restricting our selection to those with at least two outputs and a total of 2000 instances. For data preprocessing, we follow the procedure of~\citep{Grinsztajn2022-nu}.

  \begin{table}[h]
  \centering
    \begin{tabular}{llrrr}
      \toprule
      Paper & Dataset & $n$ & $p$ & $d$ \\
      \midrule
      \multirow[t]{4}{*}{\citet{Tsoumakas2011-wf}} & scm20d & 8966 & 60 & 16 \\
       & rf1 & 9005 & 64 & 8 \\
       & rf2 & 9005 & 64 & 8 \\
       & scm1d & 9803 & 279 & 16 \\
      \multirow[t]{6}{*}{\citet{Feldman2023-cc}} & meps\_21 & 15656 & 137 & 2 \\
       & meps\_19 & 15785 & 137 & 2 \\
       & meps\_20 & 17541 & 137 & 2 \\
       & house & 21613 & 14 & 2 \\
       & bio & 45730 & 8 & 2 \\
     & blog\_data & 50000 & 55 & 2 \\
      \citet{wang2023probabilistic} & taxi & 50000 & 4 & 2 \\
      \bottomrule
    \end{tabular}
    \caption{List of datasets with their characteristics.}
  \label{table:datasets}
  \end{table}

\end{document}